\documentclass{article}
\usepackage{xcolor}
\usepackage{microtype}
\usepackage{graphicx}
\usepackage{subcaption}

\usepackage{booktabs} 
\usepackage{algorithm}
\usepackage{algorithmic}
\usepackage{hyperref}
\usepackage{amsmath}
\usepackage{amssymb}
\usepackage{mathtools}
\usepackage{amsthm}
\usepackage{bm}
\usepackage{multirow}
\usepackage{enumitem}

\def\bone{\mathbf{1}}
\def\calC{\mathcal{C}}
\def\calH{\mathcal{H}}
\def\calM{\mathcal{M}}
\def\calI{\mathcal{I}}
\def\calO{\mathcal{O}}
\def\calU{\mathcal{U}}
\def\calX{\mathcal{X}}
\def\calY{\mathcal{Y}}
\def\Supp{\textup{supp}}

\def\bgamma{\bm{\gamma}}
\def\bmu{\bm{\mu}}
\def\bnu{\bm{\nu}}
\def\balpha{\bm{\alpha}}
\def\bbeta{\bm{\beta}}
\def\bG{\mathbf{G}}
\def\bC{\mathbf{C}}
\def\bg{\mathbf{g}}
\def\bp{\mathbf{p}}
\def\bq{\mathbf{q}}

\def\bw{\mathbf{w}}
\def\bx{\mathbf{x}}
\def\by{\mathbf{y}}
\def\bz{\mathbf{z}}
\def\R{\mathbb{R}}
\def\E{\mathbb{E}}
\def\u{\{u\}}
\def\bzero{\mathbf{0}}

\usepackage[capitalize,noabbrev]{cleveref}

\usepackage[accepted]{icml2024}

\def\argmax{\mathop{\rm arg\,max}\limits}
\def\argmin{\mathop{\rm arg\,min}\limits}

\theoremstyle{plain}
\newtheorem{theorem}{Theorem}[section]
\newtheorem{proposition}[theorem]{Proposition}
\newtheorem{lemma}[theorem]{Lemma}

\theoremstyle{definition}

\newtheorem{assumption}[theorem]{Assumption}
\theoremstyle{remark}

\usepackage[textsize=tiny]{todonotes}

\icmltitlerunning{Submodular Framework for Structured-Sparse Optimal Transport}

\newcommand{\commentKG}[1]{{}}
\def\S{\text{Section}} 
\def\p{\mathbf{p}}
\def\q{\mathbf{q}}

\begin{document}

\twocolumn[
\icmltitle{Submodular Framework for Structured-Sparse Optimal Transport}


\begin{icmlauthorlist}
\icmlauthor{Piyushi Manupriya}{a1}
\icmlauthor{Pratik Jawanpuria}{a2}
\icmlauthor{Karthik S. Gurumoorthy}{a3}
\icmlauthor{SakethaNath Jagarlapudi}{a1}
\icmlauthor{Bamdev Mishra}{a2}
\end{icmlauthorlist}

\icmlaffiliation{a1}{Department of Computer Science and Engineering, IIT Hyderabad, India.}
\icmlaffiliation{a2}{Microsoft, India.}
\icmlaffiliation{a3}{Walmart Global Tech, India}

\icmlcorrespondingauthor{Pratik Jawanpuria}{pratik.jawanpuria@microsoft.com}

\icmlkeywords{Machine Learning, ICML}

\vskip 0.3in
]

\printAffiliationsAndNotice{} 

\begin{abstract}
Unbalanced optimal transport (UOT) has recently gained much attention due to its flexible framework for handling un-normalized measures and its robustness properties. In this work, we explore learning (structured) sparse transport plans in the UOT setting, i.e., transport plans have an upper bound on the number of non-sparse entries in each column (structured sparse pattern) or in the whole plan (general sparse pattern). We propose novel sparsity-constrained UOT formulations building on the recently explored maximum mean discrepancy based UOT. We show that the proposed optimization problem is equivalent to the maximization of a weakly submodular function over a uniform matroid or a partition matroid. We develop efficient gradient-based discrete greedy algorithms and provide the corresponding theoretical guarantees. Empirically, we observe that our proposed greedy algorithms select a diverse support set and we illustrate the efficacy of the proposed approach in various applications.
\end{abstract}

\section{Introduction}\label{intro}

Optimal transport (OT) has emerged as a popular tool in machine learning applications for comparing probability distributions \citep{peyre2019computational}. OT computes the minimal cost to transform one distribution into another and generates a transport plan, offering a deeper understanding of the underlying geometry. The obtained transport plan may be used for aligning the support of the distributions \citep{melis18a,jawanpuria20a}, domain adaptation \citep{courty17b,nath20a}, ecological inference \citep{muzellec17a}, etc. Furthermore, OT has been explored in diverse applications such as generative modeling \citep{arjovsky17a}, shape interpolation \citep{solomon15a,han22a}, prototype selection \citep{gurumoorthy21a}, multi-label classification \citep{frogner15a,jawanpuria21a}, single-cell RNA sequencing \citep{schiebinger19a}, and hypothesis testing \citep{mmd-uot}, to name a few. 

The seminal work of \citet{sinkhorn13} popularized the entropic regularized variants of OT for their computational and generalization benefits. 
However, a notable drawback of entropic regularized OT approaches  is that they usually learn dense transport plans, where sparse (zero) entries are almost non-existent \citep{Blondel2018,liu2023sparsityconstrained}. 
Sparser transport plans are often preferred as they offer more interpretability in alignments \citep{muzellec17a,swanson-etal-2020-rationalizing}. In this regard, existing works \citep{Blondel2018,essid18a} have shown that the squared $2$-norm regularizer for OT leads to a sparse OT plan. More recently, \citet{liu2023sparsityconstrained} introduced an explicit cardinality constraint to control the sparsity level. 
It should be noted that the above works explore sparsity in the balanced OT setups, i.e., when the marginals of the transport plan are enforced to match the given distributions. 

While balanced OT is suitable for many applications, the need for robustness in the case of noisy measures motivates relaxing the marginal matching constraints \citep{frogner15a,jumbot}. This has led to several unbalanced OT (UOT) methods \citep{Liero2018,ChizatPSV18} where a KL-divergence based regularization is employed for (softly) enforcing marginal constraints. Recently, \citet{mmd-uot} proposed a maximum mean discrepancy (MMD) regularized UOT approach, termed as MMD-UOT, as an alternative to KL-regularized UOT. However, to the best of our knowledge, the existing UOT works do not focus on learning (structured) sparse transport plans with explicit cardinality constraints. 

\textbf{Contributions.} In this work, we propose novel sparsity-constrained UOT formulations. In particular, we learn UOT plans with a general sparsity constraint or a column-wise sparsity constraint. 
While the corresponding search space is non-convex and non-smooth, we identify them with well-studied matroid structures such as uniform matroid or partition matroid. Our contributions are as follows. 
\begin{itemize}[itemsep=1pt,topsep=0pt]
\item We show that the MMD-UOT problem \citep{mmd-uot}, when viewed as a function of the support set of transport plan, is equivalent to maximizing a weakly submodular function. This allows us to view our proposed sparsity-constrained UOT problems as maximizing a weakly submodular function over a  matroid constraint. 
\item We propose novel efficient gradient-based greedy algorithms (Algorithms \ref{alg:stoch-dash} and \ref{alg:colsparseOT_dash}) with attractive theoretical guarantees for maximizing a weakly submodular function over a (uniform or partition) matroid constraint. While the algorithms can be readily applied to solve our proposed constrained UOT formulations, they are also of independent interest for maximizing a general weakly submodular function.



\item A salient feature of our investigation is the dual analysis of (non-convex) weakly submodular problems. The usual approximation results corresponding to greedy maximization of (weakly) submodular functions are on  lower bounds. While these lower bounds capture the worst case performance, often in practice, they do not explain the good performance of the greedy algorithms. In this context, the duality gap analysis provides a more optimistic bound on the performance. 



\item Finally, we empirically demonstrate the effectiveness of the proposed approach in several applications.
\end{itemize}
The proofs of our theoretical results and additional experimental details are provided in the appendix sections.

\section{Preliminaries}\label{prelim}
We begin with a few notations. 
Let  $X:=\{\mathbf{x}_i\}_{i=1}^m$ and $Y:=\{\mathbf{y}_j\}_{j=1}^n$ be the source and target datasets, respectively, where $\bx_i\in\calX$ and $\by_i\in\calY$. The corresponding empirical distributions may be written as $\mu\coloneqq \sum_{i=1}^m \bmu_i \delta_{\bx_i}$ and $\nu\coloneqq \sum_{j=1}^n \bnu_i \delta_{\by_j}$, where $\bmu_i$ and $\bnu_j$ denote the mass associated with samples $\bx_i$ and $\by_j$, respectively, and $\delta_{\bz}$ represents the Dirac measure centered on $\bz$. Let $\bone$ and $\bzero$ denote a vector/matrix of ones and zeros, respectively, whose size could be understood from the context. Then, $\bmu\in\Delta_m$ and $\bnu\in\Delta_n$, where $\Delta_d=\{\bz\in\R^d_{+}:\bz^\top\bone = 1\}$.   
For $m \in \mathbb{N}$, let $[m] = \{1,2,\dots, m\}$. 
Let $V\equiv \{(i,j): i\in[m]; j\in [n]\}$ represent the index set of an $m\times n$ matrix. 
Let ${\rm vec}(\mathbf{M})$ denote the vectorization of the matrix $\mathbf{M}$, and for an index $u\equiv (i, j)$, $\mathbf{g}_u$ denotes the element $\mathbf{g}[i, j]$. 
For a non-negative vector $\bz\in\R^d_{+}$, the indices of non-zero entries in $\bz$  (its support) are denoted by the set $\Supp(\bz)=\{i\in [d]:\bz_i>0\}$. 

\subsection{Optimal Transport}\label{ot}
Optimal transport (OT) quantifies the distance between two distributions $\mu$ and $\nu$ while incorporating the geometry over their supports. Let $\bC\in \R_+^{m\times n}$ be a cost matrix induced by a cost metric $c:\calX \times \calY\mapsto\R_+$ such that $\bC_{ij}=c(\bx_i, \by_j)$. \citet{KatoroOT} proposed the OT problem between $\bmu$ and $\bnu$ as 
$\min_{\bgamma\in \Gamma(\bmu, \bnu)}\langle\bC, \bgamma\rangle,$ 
where $\Gamma(\bmu, \bnu):=\{\bgamma\in \R_+^{m\times n}: \bgamma\bone = \bmu; \bgamma^\top \bone = \bnu\}$. 
This is a \textit{balanced} OT problem due to the presence of mass preservation constraints $\bgamma\bone = \bmu$ and $\bgamma^\top \bone = \bnu$. 
The transport plan $\bgamma$ is a joint distribution with marginals $\bmu$ and $\bnu$ and supported over the (index) set $V$. 

Recent works have explored relaxing the mass-preservation constraint of classical OT for settings where measures are noisy \cite{Balaji20} or un-normalized \cite{ChizatPSV18}. Unbalanced optimal transport (UOT) replaces the constraint $\bgamma\in \Gamma(\bmu, \bnu)$ with regularizers $\mathcal{D}(\bgamma\bone, \bmu)$ and $\mathcal{D}(\bgamma^\top \bone, \bnu)$, which promote the marginals of $\bgamma$ to be close to the given $\bmu$ and $\bnu$ distributions. Here, $\mathcal{D}$ is a divergence or distance between distributions such as KL-divergence \citep{jumbot}, maximum mean discrepancy (MMD) \citep{gretton12a}, etc. 
A recent work \citep{mmd-uot} has studied MMD regularization for the UOT problem. 
Given a cost matrix $\bC$ and a universal kernel $k$ \citep{SriperumbudurFL11}, MMD-UOT \citep{mmd-uot} is the following convex problem:
\begin{equation}\label{eqn:uotmmd}
\begin{array}{l}
\qquad \qquad \min\limits_{\bgamma\geq \bzero} \ \ \calU(\bgamma), \ \rm{where} \\
\calU(\bgamma)\coloneqq \langle\mathbf{C},\bgamma\rangle 
+ \lambda_1{\rm MMD}^2_k(\bgamma\bone,\bmu) \\

\qquad \qquad \quad \quad \quad + \lambda_1{\rm MMD}^2_k(\bgamma^\top\bone,\bnu)+ \frac{\lambda_2}{2}\|\bgamma\|^2.
\end{array}
\end{equation}
Here, ${\rm MMD}_k(\bgamma\bone,\bmu) = \|\bgamma\bone-\bmu\|_{\bG_1}$, ${\rm MMD}_k(\bgamma^\top\bone,\bnu) = \|\bgamma^\top\bone-\bnu\|_{\bG_2}$, $\bG_{1}$ and $\bG_{2}$ are the Gram matrices corresponding to kernel $k$ over the source and target points, respectively, and $\|\bz\|_\bG = \sqrt{\bz^\top\bG\bz}$. We may additionally employ a squared $\ell_2$-norm regularization ($\lambda_2\geq 0$) for computational benefits \citep{Blondel2018}. 


\subsection{Submodularity} 
Submodularity is a property of set functions that exhibit diminishing returns. Given two sets $A$ and $B$ such that $A \subseteq B\subseteq V$, a set function is submodular if and only if for any $u \notin B$, $F(A\cup \u)-F(A) \geq F(B\cup \u)-F(B)$. The term, $F(A\cup \u)-F(A)$, is the marginal gain on adding an element $u$ to set the $A$ and is popularly denoted as $F(u|A)$. Likewise $F(S|B)$ denotes $F(B\cup S)-F(B)$. The set function is monotone increasing iff $F(A) \leq F(B)$ when $A \subseteq B\subseteq V$. For non-negative monotone submodular maximization problem, $\max_{S\subseteq V,|S|\leq K} F(S)$, \citet{DBLP:journals/mp/NemhauserWF78} showed that the classical greedy algorithm obtains a $(1-e^{-1})$ approximation to the optimal objective. 


Another naturally occurring structure is that of a matroid defined as follows. Given a non-empty collection $\calI\subseteq 2^V$, the pair $\calM = (V, \calI)$ is a matroid if for every two sets $A, B\subseteq V$, the following are satisfied: (i) $\emptyset\in\calI$; (ii) If $A\subseteq B$ and $B\in \calI$, then $A\in \calI$; and (iii) If $|A|< |B|$ and $A, B\in \calI$, then $\exists u\in B\setminus A$ such that $A\cup \u \in \calI$. The elements of set $\calI$ are called the independent sets of matroid $\calM$. A set $X\subseteq V$ such that $X\notin \mathcal{I}$ is called a dependent set of $\mathcal{M}$. A maximal independent set that becomes dependent upon adding any element of $V$ is called a base for the matroid. Given a matroid $\calM= (V, \calI)$, the associated matroid constraint is $S\in\calI(\calM)$, which implies that set $S$ is an independent set of $\calM$. 

A function is said to exhibit a weaker notion of submodularity, characterized by $\bm{\alpha}$-weakly submodular \cite{DasKempe18} for some $\alpha\in (0, 1]$, if $\sum_{u\in S}F(u|B)\geq \alpha.F(S|B)$ for all $S, B\subseteq V$. Similar to submodular functions, constant-factor approximation guarantees also exist for maximizing a weakly submodular set function under cardinality and matroid constraints \cite{DasKempe18, pmlr-v80-chen18b}.

\subsection{Restricted Strong Concavity and Restricted Smoothness}\label{Prel:rsc} 
On a domain $\Omega\subset \R^N \times \R^N$, a function $l:\R^N\mapsto \R$ is said to be restricted strong concave (RSC) with parameter $u_\Omega$ and restricted smooth (RSM) with parameter $U_\Omega$ 
if for all $(\bx, \by) \in \Omega$, the following holds \citep{Elenberg18}: 
$$\frac{u_\Omega}{2}\|\by-\bx\|_2^2\leq l(\bx) - l(\by)+\langle\nabla l(\bx), \by-\bx\rangle \leq \frac{U_\Omega}{2}\|\by-\bx\|_2^2.$$
We denote the RSC and RSM parameters on the domain $\Omega_K =\{ (\bx, \by): \bx \geq 0; \|\bx\|_0 \leq K;   \by \geq 0; \|\by\|_0 \leq K\}$ of all K-sparse non-negative vectors by $u_{K}$ and $U_{K}$, respectively. This set is of interest as we aim to learn non-negative transport plans with at most $K$ non-zero entries. Also, let $\tilde{\Omega} = \{(\bx,\by): \|\bx-\by\|_0 \leq 1\}$  with the corresponding smoothness parameter $\tilde{U}_1$. It can be easily verified that if $\hat{K} \leq K$, then $u_{\hat{K}} \geq u_K$ and $U_{\hat{K}} \leq U_K$ as $\Omega_{\hat{K}} \subseteq \Omega_{K}$.


\section{Proposed Method}\label{proposed-method}
Given a source $\mu$ and a target $\nu$ distributions, we now propose a novel submodular framework for structured-sparse UOT. In this regard, we generalize the MMD-UOT formulation (\ref{eqn:uotmmd}) by introducing additional (structured) sparsity constraints on the transport plan as follows:
\begin{equation}\label{eqn:sparse-uotmmd}
\min\limits_{\bgamma\in \calC}\ \calU({\bgamma}),
\end{equation}
where $\calU: \R^{m\times n}_{+}\mapsto \R_{+}$ is the function defined in (\ref{eqn:uotmmd}) and $\calC$ denotes a set of sparsity constraints. In this work, we focus on two different sparsity constraints: (a) $\calC\equiv\calC_1\coloneqq\{\bgamma\in\R^{m\times n}_{+}:\|{\rm vec}(\bgamma)\|_0\leq K_1\}$ or (b) $\calC\equiv\calC_2\coloneqq\{\bgamma\in\R^{m\times n}_{+}:\|\bgamma_j\|_0\leq K_2\ \forall\, j\in[n]\}$, where $\|\cdot\|_0$ denotes the $\ell_0$-norm and $\bgamma_j$ denotes the $j$-th column of matrix $\bgamma$. 
While $\calC_1$ imposes a cardinality constraint on the entire transport plan $\bgamma$, $\calC_2$ imposes the cardinality constraint on each column of $\bgamma$. Note that MMD-UOT formulation~(\ref{eqn:uotmmd}) is a special case of Problem (\ref{eqn:sparse-uotmmd}), e.g., when $K_1= mn$ or $K_2= m$.


Problem (\ref{eqn:sparse-uotmmd}) is non-convex over a non-smooth search space $\calC$, and hence tricky to optimize even though the objective $\calU$ is a convex function. However, we note that the constraint sets $\calC_1$ or $\calC_2$ essentially restrict the support of the transport plan $\bgamma$ to certain patterns which may be modeled using a matroid structure. 
For instance, the set $\calC_1$ may equivalently be represented as a uniform matroid $\calM_1=(V,\calI_1)$ where $\calI_1=\{S\subseteq V: |S|\leq K_1\}$. 
Similarly, the set $\calC_2$ may be equivalently modeled using a partition matroid $\calM_2=(V,\calI_2)$ where $\calI_2=\{S: S\subseteq V; |S\cap P_j|\leq K_2\, \forall\, j\in [n]\}$ with $P_j=\{(i, j): i\in [m]\}$. 
Due to this interesting correspondence between the sparsity constraints $\calC_1$ or $\calC_2$ and the matroids, we equivalently pose the continuous Problem (\ref{eqn:sparse-uotmmd}) as the following maximization problem over discrete sets representing the support of $\bgamma$:
\begin{equation}\label{eqn:reform}
\max\limits_{S\in \calI(\calM)} F(S) (\coloneqq \calU(\bzero)-\min\limits_{\bgamma:\Supp(\bgamma)\subseteq S,\bgamma\geq \bzero} \calU({\bgamma})),
\end{equation}
where the matroid $\calM$ corresponds to either the uniform matroid ($\calM=\calM_1$) or the partition matroid ($\calM=\calM_2$). Hence, we decouple the non-convex non-smooth problem (\ref{eqn:sparse-uotmmd}) into a discrete optimization problem (\ref{eqn:reform}) whose objective evaluation requires solving a convex problem. 
For a candidate set $S\in \calI(\calM)$, computing $F(S)$ essentially requires solving the MMD-UOT problem (\ref{eqn:uotmmd}) with the support of $\bgamma$ restricted to set $S$. Since the objective $\calU(\bgamma)$ is $L$-smooth, (\ref{eqn:uotmmd}) can solved using the accelerated projected gradient descent (APGD) method with a fixed step size of $1/L$ and has a linear convergence rate \cite{mmd-uot}. 

A key outcome of the above reformulation is our next result, which proves that the set function $F(\cdot)$ is weakly submodular under mild assumptions on the kernel employed in (\ref{eqn:reform}). Please refer to Appendix~\ref{kernels} for more details. 
\begin{lemma}\label{lemma:subm}
$F(.)$ is a monotone, non-negative, and $\alpha$-weakly submodular function with the submodularity ratio $\alpha\geq \frac{u_{2K}}{\tilde{U}_1}>0$, where $K$ denotes the sparsity level of the transport plan $\bgamma$. Here, $K=K_1$ for $\calM=\calM_1$ and $K=nK_2$ for $\calM=\calM_2$. 
\end{lemma}
The proof of Lemma~\ref{lemma:subm} is discussed in Appendix \ref{app-lemma:subm}. 
In the following sections, we propose efficient greedy algorithms with attractive approximation guarantees for maximizing our weakly submodular problem (\ref{eqn:reform}). 


\subsection{Learning (General) Sparse Transport Plan}\label{gen-sparse}
As discussed, sparse transport plans are more interpretable and are useful in applications such as designing topology \citep{ijcai2023p679}, word alignment \citep{arase-etal-2023-unbalanced}, etc. To this end, we consider solving (\ref{eqn:reform}) with $\calM = \calM_1$, i.e., 
\begin{equation}\label{eqn:gensparse}
\max\limits_{S\in \calI_1(\calM_1)} F(S). 
\end{equation}
This problem learns a sparse transport plan with a maximum of $K=K_1$ non-sparse and we term it as \textbf{GenSparseUOT}. 
Since (\ref{eqn:gensparse}) is a monotone, non-negative, and $\alpha$-weakly submodular maximization problem with cardinality constraint, the classical greedy method gives a constant-factor approximation guarantee of $F(S_K)\geq(1-e^{-\alpha}) {\rm OPT}$ \cite{DasKempe18}. Here, $S_K$ is the solution returned by the greedy algorithm and ${\rm OPT}$ is the optimal objective of (\ref{eqn:gensparse}). 
The classical greedy algorithm begins with an empty set $S_0 = \emptyset$ and at each iteration $i$, it finds an element $u\in V\setminus S_{i-1}$ such that the marginal gain $F(u|S_{i-1})$ is maximized. 
Hence, in the context of solving (\ref{eqn:gensparse}), the classical greedy algorithm requires solving various instances of MMD-UOT $mnK-K(K-1)/2$ times. 
The classical greedy algorithm is detailed in Algorithm~\ref{alg:gensparseOT_greedy}. 

Since the function $-\calU(\cdot)$ in the definition of $F(S)$ has RSC and RSM properties (Lemma~\ref{lemma:rsc-rsm}), we propose to employ a computationally efficient orthogonal matching pursuit (OMP) based greedy algorithm \citep{Elenberg18,gurumoorthy19a} for solving (\ref{eqn:gensparse}). 
A key feature of such strategies is that they greedily select the next  element which maximally correlates with the residual of what has already been selected, i.e., choosing the element corresponding to the largest gradient value.  
In our case, this implies solving the MMD-UOT problem (\ref{eqn:uotmmd}) for a given support set $S$ (Appendix~\ref{app:solver}) and using its solution $\bgamma$ to compute the gradient $-\nabla\calU(\bgamma_{S})$ (\ref{eqn:gradient}) for a candidate set of elements $R$. 

\begin{algorithm}[t]
\caption{Stochastic OMP algorithm for maximizing weakly submodular problems with cardinality constraint}
\begin{algorithmic}\label{alg:stoch-dash}
\STATE \textbf{Input:}  $\lambda_1, \lambda_2, \bmu, \bnu, \bC, \bG_1, \bG_2$, sparsity level $K$, $\epsilon$.
\STATE $i = 1, S_0=\emptyset, \bgamma_{S_0} = \bzero$ and $\mathbf{g} = -\nabla \calU(\bgamma_{S_0})$.
\WHILE {$i\leq K$}
\STATE \textbf{1.} Set $R_i$ as a random subset of $V\setminus S_{i-1}$ with  $mnK^{-1}\log(1/\epsilon)$ elements 
\STATE \textbf{2.} $u = \argmax_{e\in R_i}{\mathbf{g}_e}$
\STATE \textbf{3.} $S_i = S_{i-1}\cup \{u\}$
\STATE \textbf{4.} $\bgamma_{S_i} = \argmin_{\bgamma:\Supp(\bgamma)\subseteq S_i, \bgamma\ge \bzero}{\calU(\bgamma)}$ 
\STATE \textbf{5.} $\mathbf{g}= -\nabla \calU(\bgamma_{S_i})$
\STATE \textbf{6.} $i=i+1$
\ENDWHILE
\STATE \textbf{return} $S_K, \bgamma_{S_K}$
\end{algorithmic}
\end{algorithm}

In Algorithm~\ref{alg:stoch-dash}, we propose a stochastic greedy algorithm for maximizing weakly submodular problems with cardinality constraint. It employs the above discussed OMP technique for greedy selection. 
We observe that Algorithm~\ref{alg:stoch-dash} requires solving the MMD-UOT problem (\ref{eqn:uotmmd}) of size $|i|$ only once in each iteration $i$. 
Step~1 in Algorithm~\ref{alg:stoch-dash} corresponds to stochastic selection of the candidate set $R_i$ for every iteration $i$ \citep{mirzasoleiman15a}. 
The vanilla non-stochastic OMP algorithm for maximizing weakly submodular problems with cardinality constraint \citep{gurumoorthy19a} is presented in Algorithm \ref{alg:gensparseOT_dash}. 
Compared to its (non-stochastic) counterpart, Algorithm~\ref{alg:stoch-dash} is more efficient as the gradient (step~6) is computed only for a subset of the remaining elements.  
The approximation guarantee provided by Algorithm~\ref{alg:stoch-dash} is as follows. 
\begin{lemma}\label{lemma:stoch-dash}
Let $\{S_K,\bgamma_{S_K}\}$ be a solution returned by the proposed Algorithm~\ref{alg:stoch-dash}, where $S_K$ is the support of the transport plan $\bgamma_{S_K}$. Let $S^{*}$ be an optimal solution of Problem~(\ref{eqn:gensparse}). 
Then, $\mathbb{E}[F(S_K)] \geq (1- e^{-u_{2K}/\tilde{U}_1}-\epsilon) F(S^{*})$. 
\end{lemma}
The proof of Lemma~\ref{lemma:stoch-dash} is discussed in Appendix \ref{app-lemma:stoch-dash}.


\subsection{Learning Column-wise Sparse Transport Plan}\label{subsec:partitionmatroid}
We now consider learning the transport plan $\bgamma$ with column-wise sparsity constraint, i.e., every column of $\bgamma$ has at most $K_2$ non-sparse entries. 
Such an OT approach is useful in learning a sparse mixture of experts \cite{liu2023sparsityconstrained}. 
To this end, we consider solving (\ref{eqn:reform}) with $\calM = \calM_2$, i.e., 
\begin{equation}\label{eqn:colsparse}
\max\limits_{S\in \calI_2(\calM_2)} F(S).
\end{equation}
The partition matroid constraint ensures that (\ref{eqn:colsparse}) learns a transport plan in which each column has at most $K_2$ non-sparse entries. We term the proposed problem (\ref{eqn:colsparse}) as \textbf{ColSparseUOT}. The total number of non-sparse entries in the learned transport plan $\bgamma$ is $K=nK_2$.  
We note that (\ref{eqn:colsparse}) can alternatively learn row-sparse transport plans as well. 

Algorithm~\ref{alg:stoch-dash} cannot be directly employed for solving (\ref{eqn:colsparse}) as its greedy selection does not respect the partition matroid constraint. Hence, we consider the residual randomized greedy approach for matroids \cite{pmlr-v80-chen18b}, which provides a $(1+1/\alpha)^{-2}$ approximation guarantee for $\alpha$-weakly submodular maximization subject to a general matroid constraint. However, it has a high computational cost as it requires solving multiple MMD-UOT instances in each iteration. We propose a novel OMP-based greedy algorithm, Algorithm~\ref{alg:colsparseOT_dash}, for efficiently maximizing weakly submodular problems with a general matroid constraint. 

In each iteration $i$, Algorithm~\ref{alg:colsparseOT_dash} selects a uniformly random element from the best maximal independent set (base) of $\calM_2/S_{i-1}$. 
Here, $\calM_2/S = \left(V\setminus S, \calI_{\calM_2/S}\right)$ denotes the contraction of $\calM_2$ by $S$, which is a matroid on $V\setminus S$ consisting of independent sets $\calI_{\calM_2/S} :=\{I \subseteq V\setminus S: I \cup S \in \calI \}$. 
The gradient $\nabla \calU(\bgamma_{S_{i}})$ is computed via by (\ref{eqn:gradient}). It should be noted that in every iteration, the gradient needs to be computed only for the elements in $R=\{u: u\in I, I\in \calI_{\calM_2/S}\}$. 
The solution $\bgamma_{S}$ in step~4 is obtained efficiently using the APGD algorithm. It should be noted that step~1 of Algorithm~\ref{alg:colsparseOT_dash} may not require a search over all possible maximal independent sets of $\calM_2/S$. For partition matroids, step~1 essentially involves selecting the top-$(K_2-|S\cap P_j|)$ elements with the largest (thresholded) gradient values from the set $P_j\setminus S$. 
The approximation guarantee provided by our proposed Algorithm~\ref{alg:colsparseOT_dash} is as follows. 
\begin{lemma}\label{lemma:struc-dash}
Let $\{S_{K},\bgamma_{S_{K}}\}$ be the solution returned by our Algorithm \ref{alg:colsparseOT_dash}, where $S_{K}$ is the support of the transport plan $\bgamma_{S_{K}}$. Let $S^*$ be an optimal solution of (\ref{eqn:colsparse}). Then, 
$$\mathbb{E}[F(S_{K})]\geq F(S^*)\left(1+\tilde{U}_1/u_{2K}\right)^{-2}.$$
\end{lemma}
Appendix \ref{app-lemma:struc-dash} discusses the proof for Lemma~\ref{lemma:struc-dash}.
\begin{algorithm}[t]
\caption{OMP algorithm for maximizing weakly submodular problems with matroid constraint} \label{alg:colsparseOT_dash}
\begin{algorithmic}
\STATE \textbf{Input:}  $\lambda_1, \lambda_2, \bmu, \bnu, \bC, \bG_1, \bG_2$, per column sparsity level $K_2$.
\STATE $S_0=\emptyset, \bgamma_{S_0} = \bzero, K=nK_2,\mathbf{g} = -\nabla \calU(\bgamma_{S_0})$.
\FOR{$i=1,\cdots,K$}
\STATE \textbf{1.} Let $M_i$ be a maximal independent set of $\calM_2/S_{i-1}$ maximizing the sum $\sum_{u\in M_i}\max(0,\bg_u)$.
\STATE \textbf{2.} Let $u$ be a uniformly random element from $M_i$.
\STATE \textbf{3.} $S_i = S_{i-1}\cup \{u\}$
\STATE \textbf{4.} $\bgamma_{S_i} = \argmin_{\bgamma:\Supp(\bgamma)\in S_i,\bgamma\ge \bzero}{\calU(\bgamma)}$ 
\STATE \textbf{5.} $\mathbf{g} = -\nabla \calU(\bgamma_{S_{i}})$
\ENDFOR
\STATE \textbf{return} $S_{K}, \bgamma_{S_{K}}$
\end{algorithmic}
\end{algorithm}

\subsection{Gradient Computation \& Computational Cost}
\textbf{Gradient computation:} The gradient $\nabla \calU(\bgamma)$ is employed in steps 4 and 5 of both the proposed Algorithms~\ref{alg:stoch-dash}~\&~\ref{alg:colsparseOT_dash}. 
The partial gradient expression is as follows:
\begin{equation}\label{eqn:gradient}
\begin{array}{l}
\frac{\partial \calU(\bgamma)}{\partial \bgamma_{ij}} = \bC_{ij} + 2\lambda_1 \left((\bG_1)_i^\top (\bgamma \bone) +  (\bone^\top\bgamma)(\bG_2)_j\right) \\
\qquad\qquad- 2\lambda_1 \left((\bG_1)_i^\top \bmu + \bnu^\top (\bG_2)_j\right) + \lambda_2 \bgamma_{ij}.
\end{array}
\end{equation}
In (\ref{eqn:gradient}), we observe that (a) the last term $(\bG_1)_i^\top \bmu + \bnu^\top (\bG_2)_j$ is independent of $\bgamma$ and can be precomputed, and 
(b) the terms involving the full matrix $\bgamma$ decouple in $i$ and $j$. 
We leverage this structure for computing $\nabla \calU(\bgamma)$ efficiently. 

\textbf{Computational cost:} We now discuss the per-iteration computational cost of both the proposed algorithms. For a given support set $S$, both Algorithms~\ref{alg:stoch-dash}~\&~\ref{alg:colsparseOT_dash} involve solving the corresponding MMD-UOT problem to obtain the solution $\bgamma_S$. Let $R\subseteq V\setminus S$ be the set on which the gradient needs to be computed. The set $S$ is updated via greedy selection (step~2 in Algorithm~\ref{alg:stoch-dash} or steps~1\,\&\,2 in Algorithm~\ref{alg:colsparseOT_dash}) as $S\leftarrow S\cup \{u\}$, where $u\in R$ is the chosen element in the current iteration. The per-iteration cost of both the algorithms is $O(N+t\cdot M)$, where $N$ is the cost of computing the gradient of candidate elements, $t$ is the maximum iterations used for solving MMD-UOT using APGD, and $M$ is the gradient cost in every APGD iteration. The above expression does not include the one-time cost of computing matrices $\bC,\bG_1,\bG_2$ and vectors $\bG_1\bmu,\bG_2\bnu$. 

Let $I_S=\{i\in [m]:(i,j)\in S\}$, $J_S=\{j\in [n]:(i,j)\in S\}$, $I_R=\{i\in [m]:(i,j)\in R\}$, and $J_R=\{j\in [n]:(i,j)\in R\}$. 
Then, $M=\calO(|I_S|^2 + |J_S|^2 + |S|)$ and $N=\calO(|I_S||I_R| + |J_S||J_R| + |S| + |R|)$. For both the algorithms, $1\leq|I_S|,|I_R|\leq m$ and $1\leq|J_S|,|J_R|\leq n$, where the  value of these terms depend on $S$ and $R$. For Algorithm~\ref{alg:stoch-dash}, $|R|=mnK^{-1}\log(1/\epsilon)$ and for Algorithm~\ref{alg:colsparseOT_dash} with partition matroid constraint, $2\leq |R|\leq mn$.

\subsection{Dual Analysis of (\ref{eqn:sparse-uotmmd}) and (\ref{eqn:reform})}\label{subs:DualityGap}
In the previous sections, we analyzed (\ref{eqn:sparse-uotmmd}) with $\calC=\calC_1$ or $\calC=\calC_2$ using discrete submodular maximization framework, developed Algorithms~\ref{alg:stoch-dash}~\&~\ref{alg:colsparseOT_dash}, and obtained corresponding approximation guarantees (Lemma~\ref{lemma:stoch-dash} and Lemma~\ref{lemma:struc-dash}). However, (\ref{eqn:sparse-uotmmd}) may also be viewed in a continuous optimization setting. It has a convex objective but a non-convex and non-smooth constraint set. 
From this perspective, we now analyze a dual of the non-convex (\ref{eqn:sparse-uotmmd}). While only weak duality holds in our setting, the duality gap analysis may still provide insights on the closeness to optimality. 

Our next result details the primal-dual formulations corresponding to the proposed structured sparse optimal transport problem (\ref{eqn:sparse-uotmmd}) with $\calC=\calC_2$. The expressions for (\ref{eqn:sparse-uotmmd}) with $\calC=\calC_1$ can be derived likewise.
\begin{lemma}\label{lemma:dual}
    Problem (\ref{eqn:sparse-uotmmd}) with $\calC=\calC_2$ and $\lambda_2>0$ may equivalently be written as: 
    \begin{equation}\label{eqn:primal}
    \begin{array}{l}
        \min\limits_{\bgamma\geq \bzero} P(\bgamma)\big(\coloneqq \langle\mathbf{C},\bgamma\rangle + \sum_{j=1}^n \Theta(\bgamma_j)\\
        \qquad\qquad + \lambda_1(\|\bgamma\bone-\bmu\|_{\bG_1}^2+\|\bgamma^\top\bone-\bnu\|_{\bG_2}^2)\big),
    \end{array}
    \end{equation}
    where $\bgamma_j$ denotes the $j^\textup{th}$ column of $\bgamma$, $\Theta(\bgamma_j)=\frac{\lambda_2}{2}\|\bgamma_j\|^2 + \delta_{B_K}(\bgamma_j)$ and $B_K = \{\bz\in \R_{+}^m:\|\bz\|_0\leq K\}$. Here, $\delta_{B}$ is the indicator function of a set $B$ such that $\delta_{B}(\bz)=0$ if $\bz\in B$, and $\delta_{B}(\bz)=\infty$ otherwise. The following is a convex (weak) dual of the primal (\ref{eqn:primal}):
    \begin{equation}\label{eqn:dual}
    \begin{array}{l}
        \max\limits_{\balpha\in\R^m,\bbeta\in\R^n} D(\balpha,\bbeta) \big(\coloneqq\langle\balpha,\bmu\rangle + \langle\bbeta,\bnu\rangle - \frac{1}{4\lambda_1}\balpha^\top \bG_{1}^{-1}\balpha\\
        \qquad -\frac{1}{4\lambda_1} \bbeta^\top \bG_{2}^{-1}\bbeta -\sum_{j=1}^n \Theta^*(\balpha+\bbeta_j\bone - \bC_j)\big),
    \end{array}
    \end{equation}
    where $\bC_j$ denote the $j^\textup{th}$ column of $\bC$ and
    \begin{equation}\label{eqn:conjugate}
    \Theta^{*}(\bw) = \max_{\bz\in B_K} \langle\bw,\bz\rangle - \frac{\lambda_2}{2}\|\bz\|^2 .
    \end{equation}
\end{lemma}
The above result can be obtained using Lagrangian duality, and $\balpha$ and $\bbeta$ are the Lagrangian parameters corresponding to $\bgamma\bone-\bmu =\bp$ and $\bgamma^\top\bone-\bnu =\bq$ constraints, respectively, where $\bp$ and $\bq$ are auxiliary variables. To compute $\Theta^*(\bw)$, consider the permutation $\pi$ on $[m]$ such that $\bw_{\pi(i)} \geq \bw_{\pi(i+1)}$ for $1\leq i < m$. The solution is given by: $\bz_{\pi(i)} = \max\left(0, \frac{\bw_{\pi(i)}}{\lambda_2} \right)$ for $i \in [K]$, $0$ otherwise, and $\Theta^{*}(\bw)=\frac{1}{2\lambda_2}\sum_{i=1}^K (\max(0,\bw_{\pi(i)}))^2$ \citep{liu2023sparsityconstrained}. 

For a feasible primal-dual pair $\{\bgamma_S,(\balpha_S,\bbeta_S)\}$ corresponding to (\ref{eqn:primal}) and (\ref{eqn:dual}), 
$\Delta(\bgamma_S,\balpha_S,\bbeta_S) = P(\bgamma_S)-D(\balpha_S,\bbeta_S)$ is the associated duality gap. However, $\Delta(\bgamma_S,\balpha_S,\bbeta_S)$ requires computing the dual candidate $(\balpha_S,\bbeta_S)$ for the given primal candidate $\{\bgamma_S\}$, which leads to our next result. 
%
\begin{proposition}\label{prop:dualitygap}
    Let $\bgamma_S$ be a feasible primal candidate for (\ref{eqn:primal}), e.g., obtained from Algorithm~\ref{alg:colsparseOT_dash} as (\ref{eqn:primal}) and (\ref{eqn:colsparse}) are equivalent problems. Then, the dual candidate corresponding to $\bgamma_S$ is $\balpha_S = 2\lambda_1\bG_1(\bmu-\bgamma\bone)$ and $\bbeta_S = 2\lambda_1\bG_2(\bnu-\bgamma^\top\bone)$. 
\end{proposition}
Proposition~\ref{prop:dualitygap} provides concrete expressions for computing the duality gap $\Delta$. While weak duality only guarantees $\Delta\geq 0$, computing $\Delta$ may still provide an estimate of how far a candidate solution could be from optimality. 
We present the proofs of Lemma~\ref{lemma:dual} and Proposition~\ref{prop:dualitygap} in Appendix~\ref{app:dual}.

\section{Related Works}\label{subsec:related}
Since entropic-regularized OT \citep{sinkhorn13} usually learns dense transport plan, \citet{Blondel2018} proposed an alternative $\ell_2$-norm regularization for balanced OT and showed that it learns a sparse transport plan. While the degree of sparsity in $\ell_2$-norm regularized OT depends on the magnitude of the regularization parameter, it cannot be explicitly controlled as desired in several applications. 
Hence, \citet{liu2023sparsityconstrained} impose explicit column-wise sparsity constraints on the transport plan in the balanced $\ell_2$-regularized OT problem. 
To solve their $\ell_2$-regularized sparsity constrained balanced OT problem, henceforth termed as SCOT, \citet{liu2023sparsityconstrained} propose a (semi-)dual relaxation of their primal formulation in the continuous optimization setting. SCOT uses gradient updates (LBFGS or ADAM solver) to solve the (semi-)dual and requires solving (\ref{eqn:conjugate}) at each iteration. 
We note that \citet{strucOT} also leverages submodularity in the OT framework. In particular, they employ a submodular cost function. 

In contrast, we propose to learn a general or column-wise sparse transport plan in the unbalanced optimal transport (UOT) setting. 
We pose these as equivalent (weakly) submodular maximization problems under matroid (uniform or partition) constraints. Overall, we develop efficient discrete greedy algorithms to solve the primal formulation (\ref{eqn:reform}) and present corresponding approximation guarantees (Lemmas~\ref{lemma:stoch-dash}~\&~\ref{lemma:struc-dash}). The equivalence between the discrete (\ref{eqn:reform}) and the continuous (\ref{eqn:sparse-uotmmd}) problems allows us to derive a convex (weak) dual (\ref{eqn:dual}) of (\ref{eqn:reform}). 
While this dual analysis requires $\lambda_2>0$, Algorithms~\ref{alg:stoch-dash}~\&~\ref{alg:colsparseOT_dash} (and Lemmas~\ref{lemma:stoch-dash}~\&~\ref{lemma:struc-dash}) also work with $\lambda_2=0$, i.e., no additional $\ell_2$-norm regularization in (\ref{eqn:sparse-uotmmd}). On the other hand, the presence of $\ell_2$-norm regularizer is essential for SCOT \citep{liu2023sparsityconstrained}. 


\section{Experimental Results}
We evaluate the proposed approach in various applications. Experiments related to general sparse transport plans are discussed in Section~\ref{gensparseexp}, while those related to column-wise sparse transport plans are discussed in Section~\ref{colsparseexp}. 
Additional experimental results and details are presented in Appendix \ref{app:exp}. Code can be downloaded from \href{https://github.com/Piyushi-0/Sparse-UOT}{https://github.com/Piyushi-0/Sparse-UOT}.


\subsection{General Sparsity}\label{gensparseexp}
We begin with experiments where learning a sparse transport plan is desired.  

\subsubsection{Designing Topology}\label{exp-top}
Sparse process flexibility design (SPFD) aims to design a network topology that handles unpredictable demands of $n$ products by matching them to the supplies from  $m$ plants. Designing a network topology requires adding edges between the nodes that facilitate the flow of goods. A recent work \citep{ijcai2023p679} models SPFD as an OT problem. While the supplies are predefined and can be modeled as $\bmu\in [0, \infty)^m$, the demands follow a given distribution $\nu$. Hence, a set of demands $\{\bnu_i\}_{i=1}^z$ can be sampled from $\nu$, i.e., $\bnu_i\in [0, \infty)^n \sim \nu$. 
Then, the SPFD problem may be defined as \citep{ijcai2023p679}
\begin{equation}\label{eqn:spfd}
\max\limits_{ \{ \bgamma_i\in \Gamma(\bmu, \bnu_i)\}_{i=1}^z}\frac{1}{z}\sum_{i=1}^z\langle\mathbf{P}, \bgamma_i\rangle,
\ \textup{s.t. }\left\|\sum_{i=1}^z \bgamma_i\right\|_0\leq l,
\end{equation}
where $\mathbf{P}\in\R^{m\times n}$ denotes the matrix of profits (negative of the cost matrix in the OT setting) and $l$ is the total number of edges allowed in the network. 

\textbf{GSOT:} \citet{ijcai2023p679} propose a convex relaxation of the $\ell_0$-norm constraint in (\ref{eqn:spfd}) with a $\ell_1$-norm regularizer and solve the resulting OT problem, termed as group sparse OT (GSOT), using an ADMM algorithm. 
Given a solution $\{\bgamma_{i,\rm{GSOT}}\}_{i=1}^z$ of the relaxed GSOT problem, the network topology may be obtained from the aggregate solution $\bgamma_{\rm{GSOT}}=\frac{1}{z}\sum_{i=1}^z\bgamma_{i,\rm{GSOT}}$. 
The profit created by the network is approximated as $\langle\mathbf{P},\bgamma_{\rm{GSOT}}\rangle$. 
It should be noted that since the aggregate  $\bgamma_{\rm{GSOT}}$ may not satisfy $\|\bgamma_{\rm{GSOT}}\|_0\leq l$, the top-$l$ edges in $\bgamma_{\rm{GSOT}}$ which maximize the profit $\langle\mathbf{P},\bgamma_{\rm{GSOT}}\rangle$ are selected as the network topology.

\textbf{Proposed:} We propose to model the SPFD problem~(\ref{eqn:spfd}) as the following UOT problem:  
\begin{equation}\label{eqn:proposed_spfd}
\frac{1}{z}\sum_{i=1}^z\ \max\limits_{\bgamma_i\in \Gamma(\bmu, \bnu_i)}\langle\mathbf{P}, \bgamma_i\rangle,\ \textup{s.t. }\|\bgamma_i\|_0\leq \frac{l}{z},
\end{equation}
where we solve $z$ independent instances of our GenSparseUOT problem (\ref{eqn:gensparse}). Thus, we employ the proposed Algorithm~\ref{alg:stoch-dash} to solve (\ref{eqn:proposed_spfd}). Let $\{\bgamma^*_i\}_{i=1}^z$ be the obtained solution. 
The final network topology is obtained by selecting the top-$l$ significant edges of $\mathbf{P}\odot \sum_i \bgamma^*_i$, as discussed in the case of GSOT. 

\textbf{Experimental setup and results:} Using the data generation process described in \citet{ijcai2023p679}, we generate the  source and target datasets with $m=n=100$ and $z=20$. 
We compare the proposed GenSparseUOT approach only against GSOT, as other baselines such as SSOT \cite{Blondel2018} and MMD-UOT \citep{mmd-uot} do not incorporate sparsity constraint over transport plan.  
The hyperparameters of both GSOT and the proposed approach are tuned. Please refer to Appendix~\ref{app:spfd} for more details. 
In Table~\ref{table-top}, we report the expected profit (averaged across five random trials) obtained by both the approaches with $l=\{100,175,250\}$, i.e., $\max\{m, n\}\leq l \leq \rm{round}(2.5\max\{m, n\})$ \citep{ijcai2023p679}. 
Our solution is obtained via Algorithm~\ref{alg:stoch-dash} and we report our performance with different stochastic greedy parameter $\epsilon$. 
We observe that the proposed approach is significantly better than GSOT across varying network size constraints. 

\begin{table}[t]
\caption{Expected profit (higher is better) for SPFD experiment with varying network size constraint $l$. Proposed refers to our GenSparseUOT formulation (\ref{eqn:gensparse}) solved via Algorithm~\ref{alg:stoch-dash}. The result is averaged over five random trials. We observe that our approach outperforms the GSOT baseline.}
\label{table-top}
\centering
\setlength{\tabcolsep}{4pt}
\begin{tabular}{lccc}
\toprule
Method & $l=100$ & $l=175$ & $l=250$\\
\midrule
    GSOT & {0.014} & {0.031} & {0.044} \\
    Proposed ($\epsilon= 10^{-2}$) & {$0.166$} & {$ 0.224$} & {$ \mathbf{0.293}$}\\
    Proposed ($\epsilon= 10^{-3}$) & {$\mathbf{0.167}$} & {$ 0.238$} & {$0.286$}\\
    Proposed ($\epsilon= 10^{-4}$) & {$0.147$} & {$\mathbf{0.240}$} & {$0.274$}\\
\bottomrule
\end{tabular}
\end{table}

\subsubsection{Monolingual Word Alignment}\label{exp:word-alignment}
Aligning words in a (monolingual) sentence pair is an important sub-task in various natural language processing applications such as question answering, paraphrase identification, sentence fusion, and textual entailment recognition to name a few \citep{maccartney08,yao13,feldman19,brook21}. Recently, \citet{arase-etal-2023-unbalanced} employed the OT machinery to align words between given two sentences. The sentences are represented as a histogram over words and the OT cost matrix is computed using contextualized word embeddings using a (pretrained) BERT-base-uncased model \citep{devlin19}. The learned OT plan represents the alignment. \citet{arase-etal-2023-unbalanced} showed that existing OT variants 
perform at par with tailor-made word alignment techniques \citep{JaliliSabet2020SimAlignHQ}. 



\begin{figure}[t]
\centering{
\includegraphics[scale=0.4]{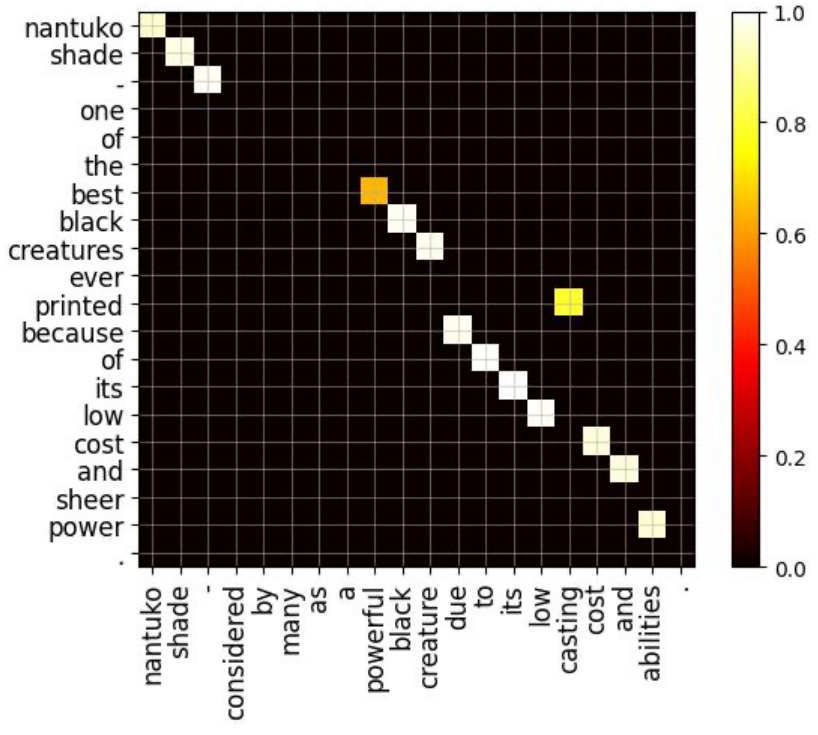}
    \caption{Example of a word alignment matrix obtained by our GenSparseUOT approach. 
    Since the sentences convey similar information, most words in either sentences have a semantic counterpart, and our approach aligns them (almost) correctly. E.g., it correctly aligns 'powerful'$\leftrightarrow$'best' and 'abilities'$\leftrightarrow$'power' and (correctly) does not map 'powerful'$\leftrightarrow$'power' even though this pair is semantically close. Words without a semantic counterpart are left unaligned (null alignment).}
    \label{App-waln}}
\end{figure}

It should be noted that words in one sentence may lack semantic counterparts in the other sentence, especially when the sentences convey different meanings. Such words correspond to \textit{null} alignments. 
Identifying null alignments is essential because it helps us reason about the semantic similarity between sentences by highlighting information inequality. 
This motivates the need of learning sparsity constrained unbalanced transport plan for such a task and we evaluate the suitability of our GenSpareUOT approach (\ref{eqn:gensparse}) for this problem. Figure~\ref{App-waln} illustrates a word alignment matrix learned by our approach for a given pair of (semantically similar) sentences.

\textbf{Experimental setup and results:} We follow the experimental setup described in \citep{arase-etal-2023-unbalanced}. The evaluation is performed on the aligned Wikipedia sentences in an unsupervised setting with the `sure' alignments, i.e., with the alignments agreed upon by multiple annotators \citep{arase-etal-2023-unbalanced}. 
Since the number of words in the input sentences is usually small, we solve GenSpareUOT (\ref{eqn:gensparse}) via Algorithm~\ref{alg:gensparseOT_dash} (which is the non-stochastic  variant of Algorithm~\ref{alg:stoch-dash}) and compare it against the OT baselines BOT, POT, KL-UOT studied by \citet{arase-etal-2023-unbalanced}, SSOT \citep{Blondel2018}, and MMD-UOT \citep{mmd-uot}. 
The hyperparameters of all methods are tuned. Please refer to Appendix~\ref{app:word_alignment} for more details. 

\begin{figure*}
\centering{
\includegraphics[scale=0.25]{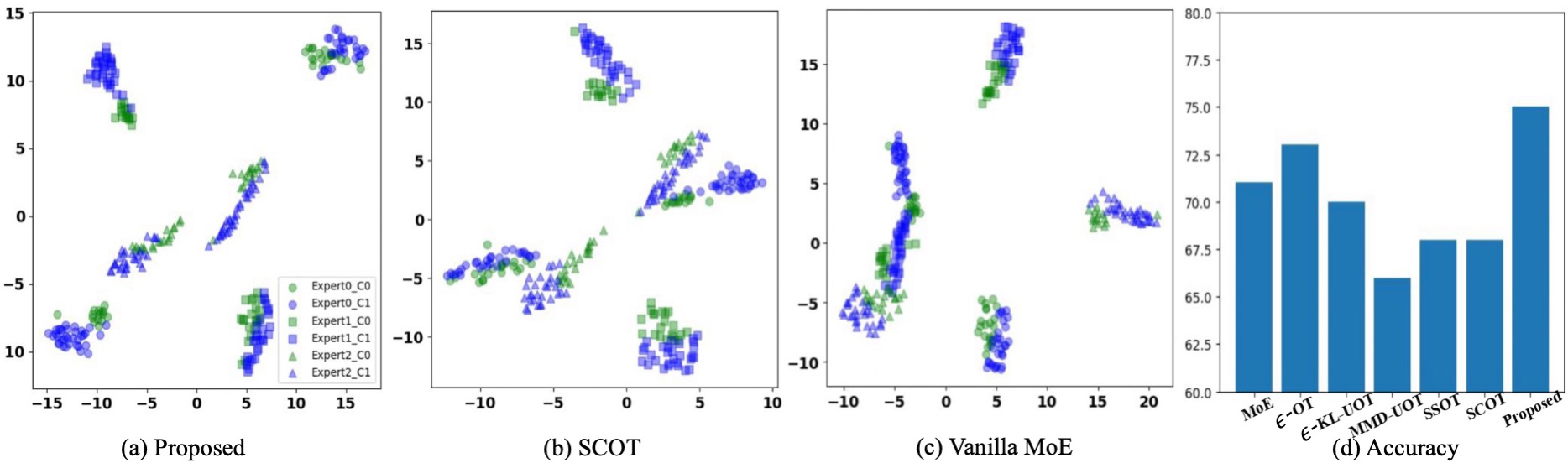}
\caption{(a)-(c) t-SNE mappings of the experts learned by different approaches. `Expert$i\_\textup{C}j$' denotes the embeddings learnt by expert $i$ for samples belonging to class $j$. The embeddings learned with the proposed approach not only distinguish the instances from the two classes but also exhibit more diversity in the knowledge acquired by every expert. (d) The accuracy obtained on the test set.}
\label{tsne-data1-main}
}
\end{figure*}

Table~\ref{table-word} reports the accuracy and the F1 scores corresponding to matching the null and the total (null + non-null) assignments. 
We see that the proposed approach is at par or better than the OT baselines studied by \citet{arase-etal-2023-unbalanced}. On the other hand, our approach outperforms MMD-UOT and the sparse OT approach SSOT. 
The corresponding precision and recall scores are detailed in Table~\ref{table-waln2}. 






\begin{table}
\caption{F1 and accuracy (Acc.) scores on the test split of the Wiki dataset. 
The scores are reported for both null and total alignments. 
Higher scores are better. The proposed approach is at par with the best performing baseline. }
\label{table-word}
\centering
\begin{tabular}{lcccc}
\toprule
 & \multicolumn{2}{c}{Null} & \multicolumn{2}{c}{Total}\\
 Method & Acc. & F1 & Acc. & F1\\
\midrule
    BOT  & {48.95} & \textbf{80.05} & {47.05} & \textbf{94.96}\\
    POT & 37.07 & 72.48 & 34.32 & 94.15\\
    KL-UOT & 44.68 & 78.71 & 42.02 & 94.63\\
    MMD-UOT & 41.35 & 75.92 & 37.74 & 93.14\\
    SSOT & 16.54 & 29.40 & 12.74 & 64.13 \\
    Proposed & \textbf{49.14} & {79.92} & \textbf{48.00} & {94.79}\\
\bottomrule
\end{tabular}
\end{table}
\subsection{Column-wise Sparse Transport Plan}\label{colsparseexp}
Mixture-of-Expert (MoE) \citep{Jacobs1991AdaptiveMO,JordanandJacob,EigenRS13} is a popular architecture that helps scale up model capacity with relatively small computational overhead. 
MoE consists of $m$ experts, which are neural networks with identical architecture, trained with a gating function (often, a shallow neural network) that routes inputs to a chosen subset of the experts. 
\citet{shazeer2017} demonstrated the utility of a sparsely-gated mixture of experts (SMoE) that selects only the top-$K_2$ experts for processing the input, where $1\leq K_2 < m$. A key motivation behind MoE/SMoE is that a complex problem may be solved by a combination of experts, each specializing on different sub-problem(s).

Given an input $\bx$, the output of SMoE is given by $\textup{SMoE}(\bx)=\sum_{r=1}^m\textup{Gate}_r(\bx)E_r(\bx)$, where $\textup{Gate}_r:\R^d\mapsto \R_+$ is the sparse gating function and $\{E_r\}_{r=1}^m$ are the experts. 
\citet{clark2022unified} proposed an entropy-regularized OT based gating function with the aim of achieving a more balanced assignment across experts. For instance, load balancing becomes crucial in distributed systems. 
Recently, \citet{liu2023sparsityconstrained} employed their SCOT method in the SMoE application, where the goal is to map each input in a batch of size $n$ to top-$K_2$ (out of $m$) experts. 
In the following, we illustrate the utility of the proposed ColSparseUOT approach (\ref{eqn:colsparse}), solved via Algorithm~\ref{alg:colsparseOT_dash}, in the SMoE setting. 

\textbf{Toy dataset.} We begin with the classification task on a toy binary dataset (with random train/test split). We train an SMoE with three (shallow) experts and a top-$2$ gating function using various approaches. The architectural and training details are provided in Appendix~\ref{colsp}. 
We first qualitatively assess the latent representations learned by  (vanilla) MoE \citep{shazeer2017}, SCOT \citep{liu2023sparsityconstrained}, and the proposed ColSparseUOT approaches. In Figure~\ref{tsne-data1-main} we show their 2-D t-SNE visualizations \citep{vandermaaten08a} on the toy dataset. 
Figure~\ref{tsne-data1-main}(a) reveals that the proposed approach's experts not only distinguish the two classes effectively but also demonstrate variety in the knowledge acquired by each expert. This is because the t-SNE maps the proposed approach's experts to well-separated locations on the 2-D plane. On the other hand, t-SNE maps the experts learned by SCOT and (vanilla) MoE approaches to overlapping/nearby regions in their respective plots shown in Figures~\ref{tsne-data1-main}(b)~\&~\ref{tsne-data1-main}(c). 
We also compare the performance of the learned SMoE models on the test split. In Figure~\ref{tsne-data1-main}(d), we report the accuracy of our proposed approach, (vanilla) MoE, SCOT, and other SMoE baselines in which the (top-$2$) gating function is based on entropy-regularized OT ($\epsilon$OT), entropy-regularized KL-UOT ($\epsilon$KL-UOT), SSOT, and MMD-UOT. We observe that our method obtains the highest accuracy. 

\textbf{CIFAR dataset.} 
We next focus on the binary classification problem of identifying whether a given image belongs to the CIFAR-10 dataset or the CIFAR-10-rotate dataset \citep{chen2022towards}. 
CIFAR-10-rotate consists of CIFAR-10 images, rotated by $30$ degrees. 
For SMoE, we consider four ResNet18 experts \cite{he2016residual} and train SMoEs with the gating network based on: (a) top-1 linear activation \citep{chen2022towards}, (b) SCOT, and (c) the  proposed ColSparseUOT (\ref{eqn:colsparse}). In both SCOT and ColSparseUOT, a sparse transport plan is learned between the $m=4$ experts and a given batch of $n$ inputs with the goal of mapping each input with only one expert ($K_2=1$). 

In Table~\ref{table-cifar}, we report the performance of all the three approaches. Since load balancing is an important aspect in MoE setup, we also report the corresponding number of inputs assigned to every expert during the inference stage. 
However, balanced allocation may not be achieved by the SMoEs by default. 
Hence, we report their results with different hyperparameters values. Please refer to Appendix~\ref{colsp} for more details on the experimental setup. 
From the results we observe that our approach obtains a good generalization performance with balanced assignments across hyperparameters. While SCOT obtains a reasonable accuracy in one case (with $\lambda=10$), its load balancing is skewed. 
For the non-OT based Top-1 MoE basline, its default setting obtains a heavily skewed allocation with two experts never getting used. In a more balanced configuration, it suffers a small accuracy drop. 
Overall, we see that our method is well suited for SMoE setting from both the generalization and load balancing points of view. 
\begin{table}[t]
\caption{Accuracy obtained on SMoE experiment along with the number of inputs allocated to each expert. We observe that the proposed approach obtains the best  generalization performance with balanced allocation across experts.}
\label{table-cifar}
\centering
\setlength{\tabcolsep}{3pt}
\begin{tabular}{l|c|cccc}
\toprule
Method & Acc. & Exp~1 & Exp~2 & Exp~3 & Exp~4\\
\midrule
    Top-1 MoE$_{\textup{default}}$ & 95.91 & 0 & 10962 & 9038 & 0\\
    Top-1 MoE$_{\textup{balanced}}$ & 93.74 & 4953 & 5018 & 4779 & 5250\\
    \midrule
    SCOT$_{\lambda=0.1}$ & 77.96 & 3341 & 1895 & 11119 & 3645\\
    SCOT$_{\lambda=10}$ & 90.56 & 1929 & 6112 & 6113 & 5846\\
    SCOT$_{\lambda=1000}$ & 56.48 & 0 & 7678 & 7788 & 4534\\
    \midrule
    Proposed$_{\lambda_1=0.1}$ & 95.56 & 5435 & 4854 & 4977 & 4734\\
    Proposed$_{\lambda_1=10}$ & 85.18 & 5000 & 5002 & 4998 & 5000\\
    Proposed$_{\lambda_1=1000}$ & 90.54 & 5000 & 5000 & 5000 & 5000\\
\bottomrule
\end{tabular}
\end{table}


\subsection{Duality Gap Comparison}\label{subsec:dualitygapcomparison}
In this section, we compare the optimization quality of the proposed Algorithm~\ref{alg:colsparseOT_dash} and the SCOT algorithm \citep{liu2023sparsityconstrained} in solving our sparse UOT problem with the column-wise sparsity constraint (\ref{eqn:primal}). In Section~\ref{subsec:partitionmatroid}, we propose to solve (\ref{eqn:primal}) in discrete optimization setting, via an equivalent reformulation (\ref{eqn:reform}). 
It should be noted that \citet{liu2023sparsityconstrained} study a $\ell_2$-regularized \textit{balanced} OT problem with column-wise sparsity constraint. They propose a gradient descent based algorithm (e.g., LBFGS) with sparse projections to optimize a dual relaxation of their primal problem. We use their algorithm to solve (\ref{eqn:dual}), which is a dual of (\ref{eqn:primal}). 
While Algorithm~\ref{alg:colsparseOT_dash} learns a primal solution $\bgamma_1$ of (\ref{eqn:primal}), the corresponding dual solution $\{\balpha_1,\bbeta_1\}$ can be obtained via Proposition~\ref{prop:dualitygap}. 
SCOT, on the other hand, obtains a dual solution $\{\balpha_2,\bbeta_2\}$ of (\ref{eqn:dual}) and then obtains the corresponding primal solution $\bgamma_2$ by solving the sparse projection problem (\ref{eqn:conjugate}). Hence, we can compute and compare the duality gap $\Delta(\bgamma,\balpha,\bbeta) = P(\bgamma)-D(\balpha,\bbeta)$ associated with the solutions obtained by both the algorithms. 

\textbf{Experimental setup and results:} The source and target measures are taken to be the empirical measures over two randomly chosen $100$-sized batches of CIFAR-10. We compare the duality gap over 
a range of hyperparameter $(\lambda_1,\lambda_2)$ values and different kernels employed for the MMD computation in (\ref{eqn:sparse-uotmmd}). The kernel hyperparameters are fixed according to the median heuristics \citep{gretton12a}. 

Table~\ref{table: main-paper-DG-imq-v2} reports the results with a inverse multiquadratic kernel \citep{SriperumbudurFL11}. We observe that our approach outperforms SCOT by obtaining at least three times lower duality gap. In a couple of cases, the duality gap associated with our Algorithm~\ref{alg:colsparseOT_dash} is $< 10^{-10}$, signifying that it has converged at (or very close to) a global optimum. Additional results are discussed in Appendix~\ref{app:dualitygap}. 

\begin{table}[t]
\caption{Duality gap ($\Delta$) comparison for solving (\ref{eqn:primal}) with various hyperparameters. Lower duality gap is better. We observe that our approach obtains significantly lower duality gap than SCOT.}
\label{table: main-paper-DG-imq-v2}
\centering{
\setlength{\tabcolsep}{4pt}
\begin{tabular}{rr|cc|cc}
\toprule
\multirow{2}{*}{$\lambda_1$} & \multirow{2}{*}{$\lambda_2$} & \multicolumn{2}{c|}{Proposed solver} & \multicolumn{2}{c}{SCOT solver}\\
 &  & Primal obj. & $\Delta$ & Primal obj. & $\Delta$ \\ 
\midrule
0.1 & 0.1 & $\mathbf{0.02993}$ & $\mathbf{<10^{-10}}$	& 0.03169	& 0.00232 \\
1 & 0.1 & $\mathbf{0.09183}$	& $\mathbf{0.01911}$	& 0.27172		& 0.19111 \\
10 & 0.1 & $\mathbf{0.11682}$ & $\mathbf{0.64896}$	& 2.30889	& 2.21029 \\
0.1 & 1 & $\mathbf{0.03036}$	& $\mathbf{<10^{-10}}$	& 0.03116	& 0.00114 \\
1 & 1 & $\mathbf{0.09409}$	& $\mathbf{0.00286}$	& 0.10371	&	0.01216 \\
10 & 1 &  $\mathbf{0.11897}$	& $\mathbf{0.05468}$	& 0.32334	& 0.21289\\ 
\bottomrule
\end{tabular}
}
\end{table}

\section{Conclusion}
In this work we proposed sparsity-constrained unbalanced OT formulations and presented an interesting viewpoint of the problem as that of maximization of a weakly submodular function over a uniform or partition matroid. To this end, we propose novel greedy algorithms having attractive approximation guarantees. A duality gap analysis further provides an empirical way of validating the optimality of our greedy solution.
Experiments across different applications shows the efficacy of the proposed approach. At a conceptual level, our work shows a novel connection between OT and submodularity.  
A future work could be to expand on the variants of structured sparsity patterns in the OT plan.

\section*{Acknowledgements}
PM and JSN acknowledge the support of Google PhD Fellowship and Fujitsu Limited (Japan), respectively.

\section*{Impact Statement}
This paper tries to advance the field of optimal transport and its applications to machine learning. There are many potential societal consequences of our work, none of which we feel must be specifically highlighted here.

\bibliography{main}
\bibliographystyle{icml2024}

\newpage
\appendix
\renewcommand{\thesection}{A\arabic{section}}
\renewcommand{\thesubsection}{A\arabic{section}.\arabic{subsection}}
\renewcommand{\thefigure}{\textit{A}\arabic{figure}}
\renewcommand{\thetable}{\textit{A}\arabic{table}}
\renewcommand{\theequation}{A\arabic{equation}}
\renewcommand{\thealgorithm}{A\arabic{algorithm}}

\newtheorem{manualtheoreminner}{Theorem}
\newenvironment{manualtheorem}[1]{%
  \renewcommand\themanualtheoreminner{#1}%
  \manualtheoreminner
}{\endmanualtheoreminner}

\newtheorem{manuallemmainner}{Lemma}
\newenvironment{manuallemma}[1]{%
  \renewcommand\themanuallemmainner{#1}%
  \manuallemmainner
}{\endmanuallemmainner}

\newtheorem{manualcorrinner}{Corollary}
\newenvironment{manualcorr}[1]{%
  \renewcommand\themanualcorrinner{#1}%
  \manualcorrinner
}{\endmanualcorrinner}

\onecolumn
\section{Background}

\subsection{Weak Submodularity}\label{weak-sub}
\citet{DasKempe18} defined the notion of approximate submodularity governed by a submodularity ratio. For a monotone function $F$ the submodularity ratio, w.r.t. a set $S$ and a parameter $K\geq 1$, is defined as follows.
$\alpha_{L, K}(F)=\min\limits_{S\subseteq L, A:|A|\leq K, A\cap S=\phi}\frac{\sum_{u\in A}F(S\cup \u)-F(S)}{F(S\cup A)-F(S)},$ with $0/0:=1$. $F$ is submodular iff $\alpha_{S, K}\geq 1$. If the ratio $\alpha\equiv \frac{\sum_{u\in A}F(S\cup \u)-F(S)}{F(S\cup A)-F(S)}$ is greater than 0 and not necessarily greater than 1, then $F$ is $\alpha$-weakly submodular.

\subsection{Characteristic Kernel, Universal Kernel, and Maximum Mean Discrepancy (MMD)}\label{kernels}

We have the following assumption on the kernel corresponding to the MMD regularization in (\ref{eqn:reform}):
\begin{assumption}\label{assmp:kernel}
    The kernel $k$ corresponding to the MMD regularizations in (\ref{eqn:reform}) is bounded and universal. 
\end{assumption}
In the following, we briefly discuss the above concepts. 

\textbf{Boundedness:}
A kernel $k: \mathcal{X}\times \mathcal{X} \mapsto \R$ is said to be bounded if $k(\bx, \by)<\infty, \forall \bx, \by\in \mathcal{X}$. In the continuous domain, examples of bounded kernels include the RBF (Gaussian) kernel or the IMQ (inverse multiquadratic) kernel. 

\textbf{Kernel mean embeddings:} 
Let $\phi(\cdot)$ and $\calH$ be the canonical feature map and the canonical reproducing kernel Hilbert space (RKHS) corresponding to the kernel $k$. The kernel mean embedding \cite{Muandet_2017} of a random variable $X\sim\mathcal{P}$ is defined as $\psi_{\mathcal{P}}\coloneqq \mathbb{E}_{X\sim\mathcal{P}}[\phi(X)]$. If the kernel $k$ is bounded, then $\psi_X\in \calH$ and is well defined. 

\textbf{Characteristic and universal kernels:} Characteristic kernels \citep{SriperumbudurFL11} are those for which the map $\mathcal{P}\mapsto \psi_{\mathcal{P}}$ is injective (one-to-one). 
A kernel defined over a domain $\calX$ is universal if and only if its RKHS is dense in the set of all continuous functions over $\calX$. All universal kernels are also characteristic kernels (over their respective domains). 
Examples of universal kernels include the Kronecker delta kernel for discrete measures, the Gaussian kernel for continuous measures, the IMQ kernel for continuous measures, etc.

\textbf{Maximum mean discrepancy (MMD):} Given a characteristic kernel $k$, and distributions $\mu$ and $\nu$, the MMD metric between $\mu$ and $\nu$ is defined as \citep{gretton12a}
\begin{equation}
    {\rm MMD}_k(p,q) = \|\psi_{\mu} - \psi_{\nu}\|_{\calH} = \max_{f:\|f\|_\calH\leq 1} \langle f, \psi_{\mu} \rangle_{\calH} - \langle f, \psi_{\nu}  \rangle_{\calH},
\end{equation}
where $\langle\cdot,\cdot\rangle_\calH$ and $\|\cdot\|_\calH$ denote the RKHS inner product and RKHS norm corresponding to the kernel $k$, respectively. 


\section{Proofs on the theoretical results presented in the main paper}\label{app:theory}

\subsection{A few useful properties of $F(S)$ as defined in (\ref{eqn:reform})} 

Let the function $F(S)$ be as defined in (\ref{eqn:reform}), i.e., $F(S) \coloneqq \calU(\bzero)-\min\limits_{\bgamma:\Supp(\bgamma)\subseteq S,\bgamma\geq \bzero} \calU({\bgamma})=\calU(\bzero) - \calU(\bgamma_S)$. Here, $\calU: \R^{m\times n}_{+}\mapsto \R_{+}$ and $\bgamma_S$ denotes an optimal solution of the convex MMD-UOT problem (\ref{eqn:uotmmd}) with a given fixed support $S$. In the following, we first prove that the function $-\calU(\cdot)$ has a finite RSC and RSM parameters (Section~\ref{Prel:rsc}) and then use this property to prove a couple of results corresponding to $F(S)$. 

\begin{lemma}\label{lemma:rsc-rsm}
$-\calU(\cdot)$ has a finite restricted strong concavity (RSC) parameter $(u_\Omega)$ and a finite restricted  smoothness (RSM) parameter $(U_\Omega)$ whenever the employed kernel function $k$ is universal. 
\end{lemma}
\begin{proof} We first prove that $-\calU(.)$ has a finite RSC, RSM parameters for the case $\lambda_2=0$. Given $\bgamma,\bgamma'\in \R^{m\times n}$, we have the following result 
\begin{equation}\label{proof-rsc-rsm}
\begin{array}{l}
-\left(\calU(\bgamma)  - \calU(\bgamma')-\langle\nabla \calU(\bgamma), \bgamma-\bgamma'\rangle\right)\\
\qquad\qquad\qquad =\lambda_1\left( (\bgamma\bone_n-\bgamma'\bone_n)^\top\bG_1(\bgamma\bone_n-\bgamma'\bone_n)+(\bgamma^\top\bone_m-\bgamma'^\top\bone_m)^\top\bG_2(\bgamma^\top\bone_m-\bgamma'^\top\bone_m)\right)\\
\qquad\qquad\qquad =\lambda_1\left( \rm{Tr}\left((\bgamma-\bgamma')^\top \bG_1 (\bgamma-\bgamma')\bone_n\bone_n^\top\right) + \rm{Tr}\left((\bgamma^\top-\bgamma'^\top)^\top \bG_2(\bgamma^\top-\bgamma'^\top)\bone_m\bone_m^\top\right) \right).
\end{array}
\end{equation}
where the function $\calU(\bgamma)$ and its gradient are defined in (\ref{eqn:uotmmd}) and (\ref{eqn:gradient}). $\rm{Tr}(\cdot)$ denotes the trace operator. 

Let $e^1_0,~e^1_1$ denote the minimum and maximum eigenvalues of $\bG_1$. Let $e^2_0,~e^2_1$ denote the minimum and maximum eigenvalues of $\bG_2$. Then (\ref{proof-rsc-rsm}) implies,
$ \|\bgamma-\bgamma'\|^2 \lambda_1(e^1_0n+e^2_0m)\leq ~-\left(\calU(\bgamma) - \calU(\bgamma')-\langle\nabla \calU(\bgamma), \bgamma'-\bgamma\right\rangle)~\leq \lambda_1(e^1_1n + e^2_1m)\|\bgamma-\bgamma'\|^2$. Thus, the RSC constant becomes $u_\Omega=\lambda_1(e^1_0n+e^2_0m)$ and the RSM constant becomes $U_\Omega=\lambda_1(e^1_1n + e^2_1m)$. We recall that all characteristic kernels are universal (\ref{kernels}). We use that the gram matrices of universal kernels are full-rank \citep[Corollary~32]{Song08}. Hence, with a characteristic kernel (like the Gaussian kernel or the inverse multi-quadratic kernel), $u_\Omega \textup{ and } U_\Omega>0$. Invoking the Gershgorin circle theorem, the maximum eigenvalue of the gram matrices can be upper-bounded by the maximum row sum, which is finite for bounded kernels (like the Gaussian kernel or the inverse multi-quadratic kernel). We have that $0<u_\Omega \leq U_\Omega<\infty$. 

It can be easily seen that, when $\lambda_2>0,$ the RSM constant becomes $\lambda_1(e_1^1n+e_1^2m)+\lambda_2/2$ and the RSC constant becomes $\lambda_1(e_0^1n+e_0^2m)+\lambda_2/2$.
\end{proof}

\begin{lemma}\label{lemma2.1}
$F(S\cup \u) - F(S)\geq \frac{1}{2\tilde{U}_1}\left(\mathbf{g}^+_u(\bgamma_S)\right)^2$, where $\mathbf{g}^+_u(.)\equiv \max\{-\nabla \calU_u(.),\ 0\}$.
\end{lemma}
\begin{proof}
Let $\bone^{\u}\in \R^{m\times n}$ denote a matrix of zeros with 1 at the index given by $\u$. Let $\by^{\u}\equiv\bgamma_S+\eta \bone^{\u}$ for some $\eta\geq 0$.
\begin{align*}
F(S\cup \u)-F(S) &= -\calU(\bgamma_{S\cup \u})+\calU(\bgamma_S)\\
& \geq -\calU(\by^{\u}) + \calU(\bgamma_S)\\
& \geq \langle-\nabla \calU(\bgamma_S), \eta \bone^{\u} \rangle-\frac{\tilde{U}_1}{2}\eta^2.
\end{align*}
On maximizing wrt $\eta\geq 0,$ we get $F(S\cup \u) - F(S)\geq \frac{1}{2\tilde{U}_1} \left( \mathbf{g}_u^+(\bgamma_S)\right)^2 \ \left(\textup{when }\eta = \frac{\mathbf{g}^+_u(\bgamma_S)}{\tilde{U}_1}\right)$.
\end{proof}

\begin{lemma}\label{lemma2.2}
$F(S\cup A) - F(S)\leq \frac{1}{2u_{\bar{m}}}\sum_{u\in A}\left(\mathbf{g}_u^+(\bgamma_S)\right)^2$, where $\mathbf{g}_u^+(.)\equiv \max\{-\nabla \calU_u(.),\ 0\}$ and $\bar{m} = |S|+|A|$.
\end{lemma}
\begin{proof}
As $\bgamma_S, ~ \bgamma_{S\cup A}$ are the minimizes, $\calU(\bzero_{m\times n})-\calU(\bgamma_S)= F(S)$ and $\calU(\bzero_{m\times n})-\calU(\bgamma_{S\cup A})=F(A\cup S)$. We now upper-bound, $F(S\cup A)-F(S)=\calU(\bgamma_S)-\calU(\bgamma_{S\cup A})$.

Using the RSC and RSM constants of $-\calU$ (\ref{lemma:rsc-rsm}), we have the following.
\begin{align}\label{lemmaA1.2}
&\frac{u_{\bar{m}}}{2}\|\bgamma_{S\cup A}-\bgamma_S\|^2 \leq -\calU(\bgamma_S) + \calU(\bgamma_{S\cup A}) + \langle-\nabla \calU(\bgamma_S), \bgamma_{S\cup A}-\bgamma_S \rangle\nonumber\\
\implies & 0\leq -\calU(\bgamma_{S\cup A})+\calU(\bgamma_S) \leq \langle -\nabla \calU(\bgamma_S), \bgamma_{S\cup A}-\bgamma_S \rangle -\frac{u_{\bar{m}}}{2}\|\bgamma_{S\cup A}-\bgamma_S\|^2 \nonumber\\
&\qquad \qquad \qquad \qquad ~\qquad \leq \max_{\mathbf{W}:\mathbf{W}\geq \bzero_{m\times n}, \mathbf{W}_{V\setminus(S\cup A)}=\bzero} \ \langle-\nabla \calU(\bgamma_S), \mathbf{W} - \bgamma_S\rangle -\frac{u_{\bar{m}}}{2}\|\mathbf{W}-\bgamma_S\|^2.
\end{align} 
The matrix $\mathbf{W}^*$ that attains the maximum is described as follows. $\mathbf{W}^*_{S\cup A}=\max\left\{\frac{1}{u_{\bar{m}}}(-\nabla \calU_{S\cup A}(\bgamma_S))+(\bgamma_S)_{S\cup A}, \bzero \right\}$. Now, from the KKT conditions, we have that $\forall j\in S$,
$$
(\bgamma_S)_j > 0 \implies -\nabla \calU_j(\bgamma_S)=0 \textup{ and }(\bgamma_S)_j = 0 \implies -\nabla \calU_j(\bgamma_S)\leq 0.
$$
Hence, $(\mathbf{W}^*-\bgamma_S)_j=0~\forall j\in S$. Also, $(\bgamma_S)_j=0\ \forall j\in A$. Thus, $\forall j\in A, \left(\mathbf{W}^*-\bgamma_S\right)_j=\max \left\{\frac{1}{u_{\bar{m}}}\left(-\nabla \calU_j(\bgamma_S)\right), 0 \right\}$. Using this in (\ref{lemmaA1.2}), we get 
$$
F(S\cup A)-F(S)\leq \frac{1}{2u_{\bar{m}}}\sum_{u\in A}\left(\mathbf{g}_u^+(\bgamma_S)\right)^2,
$$ where $\mathbf{g}_u^+(.)\equiv \max\{-\nabla \calU_u(.),\ 0\}$ and $\bar{m} = |S|+|A|$.
\end{proof}

\subsection{Proof of Lemma \ref{lemma:subm}}\label{app-lemma:subm}

\begin{proof}
We have that, $F(S) \equiv \calU(\bzero_{m\times n})-\min\limits_{\bgamma:\Supp(\bgamma)\in S,~ \bgamma\ge \bzero_{m\times n}} \calU({\bgamma})$. From the definition of $\min$, it follows that $F(.)$ is a monotonically increasing function of $S$, i.e., if $S_1\subseteq S_2\subseteq V$, then $F(S_1)\leq F(S_2)$. As $F(.)$ is monotonically increasing, $F(S)\geq F(\phi)=\calU(\bzero_{m\times n})-\calU(\bzero_{m\times n})=0$. This shows the non-negativity of $F(.)$. From Lemma \ref{lemma:rsc-rsm}, we know that $\calU(.)$ has a finite RSC and RSM constants: $u_\Omega \textup{ and } U_\Omega$ respectively. Now, the proof of $\alpha$-weak submodularity of $F(.)$ follows the proof technique used in \citet[Theorem $IV.3$]{gurumoorthy19a}.

For weak-submodularity (Appendix~\ref{weak-sub}), we need to lower bound $F(S\cup \u)-F(S)$ and upper bound $F(S\cup A)-F(S)$. Let $\bar{m}=|S|+|A|$. From Lemma \ref{lemma2.1}, we have that, $F(S\cup \u)-F(S)\geq \frac{1}{2\tilde{U}_1}\left( \bg^+_u(\bgamma_S)\right)^2$. From Lemma \ref{lemma2.2}, we have that, $0\leq F(S\cup A)-F(S)\leq \frac{1}{2u_{\bar{m}}}\sum_{u\in A}\left(\bg^+_u(\bgamma_S) \right)^2.$ Using these, the ratio $\frac{\sum_{u\in A}F(S\cup \u)-F(S)}{F(S\cup A)-F(S)}\geq \frac{u_{\bar{m}}}{\tilde{U}_1}.$ We now lower bound $u_{\bar{m}}$. We recall that $\bar{m}:=|S|+|A|$ for $S, A\subseteq V$. With the general sparsity constraints on the support (Section \ref{gen-sparse}), $\bar{m}\leq 2K_1$, which makes $u_{\bar{m}}\geq u_{2K_1}$ (Section \ref{Prel:rsc}). With the column-wise sparsity constraints on the support, (Section \ref{subsec:partitionmatroid}), $\bar{m}\leq 2nK_2$, which makes $u_{\bar{m}}\geq u_{2nK_2}$ (Section \ref{Prel:rsc}). Hence, we proved that $F(\cdot)$ is $\alpha$-weakly submodularity with $\alpha\geq \frac{u_{2K_1}}{\tilde{U}_1}$ for the general sparsity case and $\alpha\geq \frac{u_{2nK_2}}{\tilde{U}_1}$ for the column-wise sparsity constraints. Combining the two cases, we have that $\alpha\geq \frac{u_{2K}}{\tilde{U}_1}$ where $K$ denotes the sparsity level of the transport plan.
\end{proof}


\subsection{Proof of Lemma \ref{lemma:stoch-dash}}\label{app-lemma:stoch-dash}
\begin{proof}
Let $S^*$ be the optimal support set. Let $V$ denote the ground set of cardinality $N\equiv m\times n$ and $K=K_1$ as the general sparsity cardinality constraint. Let $S_i$ be the subset chosen by Algorithm \ref{alg:stoch-dash} up to iteration $i$. We define $\mathbf{g}^+_j(\bgamma_{S_i})\equiv \max\{-\nabla \calU_j(\bgamma_{S_i}),\ 0\}$.

Let a randomly chosen set $R$ consist of $s=\frac{N}{K}\log{\frac{1}{\epsilon}}$ elements from $V\setminus S_i$. We first estimate the probability that $R\cap (S^*\setminus S_i)$ is non-empty.

\begin{align}\label{app:stoch0}
\Pr\left[R\cap (S^*\setminus {S_i})\neq\phi\right] & = 1-\Pr\left[R\cap (S^*\setminus {S_i})=\phi\right] \nonumber\\
& = 1-\left(1-\frac{|S^*\setminus {S_i}|}{|V\setminus {S_i}|} \right)^s \nonumber\\
& \geq 1-e^{-s\frac{|S^*\setminus {S_i}|}{|V\setminus {S_i}|}} \ (\because 1-x\leq e^{-x}) \nonumber\\
& \geq 1-e^{-s\frac{|S^*\setminus {S_i}|}{N}} \ (\because |V\setminus {S_i}|\leq N) \nonumber\\
& \stackrel{(1)}{\geq} \left(1-e^{\frac{sK}{N}}\right)\frac{|S^*\setminus {S_i}|}{K} \ \left(\textup{Using concavity of }f(x)=1-e^{-\frac{s}{N}x}.\right) \nonumber\\
& = (1-\epsilon)\frac{|S^*\setminus {S_i}|}{K} \ \textup{(Substituting the value of }s.)
\end{align}
The inequality (1) is detailed as follows. As $f(x) = 1-e^{-\frac{s}{N}x}$ is a concave function for $x\in \R$ and as $\frac{|S^*\setminus {S_i}|}{K}\in [0, 1]$, we have that $f\left(\frac{|S^*\setminus {S_i}|}{K} K+ \left(1-\frac{|S^*\setminus {S_i}|}{K}\right).0\right)\geq \frac{|S^*\setminus {S_i}|}{K} f(K) + \left(1-\frac{|S^*\setminus {S_i}|}{K}\right)f(0)$.

Now, we observe that, for an element $u$ to be picked by Algorithm \ref{alg:stoch-dash}, $\mathbf{g}^+_u(\bgamma_{S_i})\geq \mathbf{g}^+_b(\bgamma_{S_i}),\ \forall b\in R\cap \left(S^*\setminus S_i\right)$ (if non-empty). We have that,
\begin{align}\label{app:stoch1}
\mathbf{g}^+_u(\bgamma_{S_i})\geq \mathbf{g}^+_b(\bgamma_{S_i}) &\implies \left(\mathbf{g}^+_u(\bgamma_{S_i})\right)^2 
\geq \left(\mathbf{g}^+_b(\bgamma_{S_i})\right)^2 \nonumber\\
&\implies \E\left[\left(\mathbf{g}^+_u(\bgamma_{S_i})\right)^2|S_i\right] \geq \E\left[\left(\mathbf{g}^+_b(\bgamma_{S_i})\right)^2|S_i\right] \Pr\left[ R\cap \left(S^*\setminus S_i\right)\neq \phi \right]
\end{align}
for any element $b\in R\cap \left(S^*\setminus S_i\right)$ (if non-empty).

We then use that $R$ is equally likely to contain each element of $S^*\setminus S_i$, so a uniformly random element of $R\cap \left(S^*\setminus S_i\right)$ is a uniformly random element of $S^*\setminus S_i$. From (\ref{app:stoch1}), 
\begin{align}\label{app:stoch2}
\E\left[\left(\mathbf{g}^+_u(\bgamma_{S_i})\right)^2|S_i\right] & \geq \Pr\left[ R\cap \left(S^*\setminus S_i\right)\neq \phi \right]\frac{1}{|S^*\setminus S_i|}\sum_{b\in S^*\setminus S_i}\left(\mathbf{g}^+_b(\bgamma_{S_i})\right)^2\nonumber\\
& \geq \frac{1-\epsilon}{K} \sum_{b\in S^*\setminus S_i}\left(\mathbf{g}^+_b(\bgamma_{S_i})\right)^2 \ \left(\textup{From \ref{app:stoch0}.}\right)
\end{align}
With $S_{i+1}= S_i\cup \u$, Lemma \ref{lemma2.1} gives us,
\begin{equation}\label{app:stoch3}
\E\left[2\tilde{U}_1\left(F(S_{i+1})- F(S_i)\right)|S_i\right]\geq \E\left[ \left(\mathbf{g}^+_u(\bgamma_{S_i}) \right)^2|S_i \right].
\end{equation}

Using Lemma \ref{lemma2.2}, we have that,
\begin{equation}\label{app:stoch4}
\frac{1-\epsilon}{K}\sum_{b\in S^*\setminus S_i} \left( \mathbf{g}^+_b(\bgamma_{S_i}) \right)^2\geq \frac{2u_{\bar{m}}(1-\epsilon)}{K} \left( F(S^*)-F(S_i)\right)\geq \frac{2u_{2K}(1-\epsilon)}{K} \left( F(S^*)-F(S_i)\right).
\end{equation}
The last inequality uses that $\bar{m} = |S_i| + |S^*\setminus S_i|\leq 2K$ and hence $u_{\bar{m}} \geq u_{2K}$.

From inequalities (\ref{app:stoch2}), (\ref{app:stoch3}),  and (\ref{app:stoch4}), we get the following.
\begin{align*}
& \E\left[2\tilde{U}_1\left(F(S_{i+1})- F(S_i)\right)|S_i\right] \geq \frac{2u_{2K}(1-\epsilon)}{K} \left( F(S^*)-F(S_i)\right)\\
\implies & \E\left[\left(F(S_{i+1})- F(S_i)\right)|S_i\right] \geq \frac{u_{2K}(1-\epsilon)}{K\tilde{U}_1} \left( F(S^*)-F(S_i)\right) \ \left(\because \frac{u_{2K}}{\tilde{U}_1}\in(0, 1] \right)\\
\implies & \E\left[F(S_{i+1})- F(S_i)\right] \geq \frac{u_{2K}(1-\epsilon)}{K\tilde{U}_1} \left( F(S^*)-\E\left[F(S_i)\right]\right) \ (\textup{Taking an expectation over }A_i.)
\end{align*}
On re-arranging and using induction, we get
\begin{align*}
\E\left[F(S_K)\right]&\geq \frac{u_{2K}(1-\epsilon)}{\tilde{U}_1K} F(S^*)\left( \sum_{i=0}^{K-1} \left( 1-\frac{u_{2K}(1-\epsilon)}{K\tilde{U}_1} \right)^i \right)\\
    &\geq \left(1-\left( 1-\frac{u_{2K}(1-\epsilon)}{K\tilde{U}_1} \right)^K \right)F(S^*) \ \left(\textup{We use }\frac{u_{2K}}{\tilde{U}_1}\in (0, 1] \textup{ and sum the Geometric series.} \right) \\
    & \geq \left(1-e^{-\frac{u_{2K}(1-\epsilon)}{\tilde{U}_1}} \right)F(S^*) \ (\textup{Using }e^{-x}\geq 1-x) \\
    & = \left(1-e^{-r(1-\epsilon)} \right)F(S^*) \ \left(\textup{where }r=\frac{u_{2K}}{\tilde{U}_1}\right)
    \\
    & \geq \left(1-e^{-r}-\epsilon\right)F(S^*).
\end{align*}
The last inequality is detailed as follows. Let us first consider a function $f$ over the domain $[0,  1]$ defined as $f(x) = z^x-xz$ for some $z\geq 0$. This is a convex function with $f(0)=1, f(1)=0$. Thus, $z^x-xz\leq 1$. Taking $z=e^r$ proofs the result. The proof technique is inspired by the proof for the stochastic-greedy algorithm \cite{mirzasoleiman15a}.
\end{proof}

\subsection{Proof of Lemma \ref{lemma:struc-dash}}\label{app-lemma:struc-dash}
Our proof is inspired by the approximation ratio proof in \citet{pmlr-v80-chen18b}.  We first discuss the following lemma where our proof differs from that in \citet{pmlr-v80-chen18b}. In this subsection, we use $K$ to denote the overall cardinality constraint of the column-wise sparse transport plan, i.e., $K=nK_2$.
\begin{lemma}\label{matroid} For every $1\leq i \leq K, ~ \E[F(S_i)]\geq \E[F(S_{i-1})]+\alpha \frac{\E[F(OPT_{i-1}\cup S_{i-1})]-\E[F(S_{i-1})]}{K-i+1}$, where $\alpha=\frac{u_{2K}}{\tilde{U}_1}$.
\end{lemma}
\begin{proof}

The base $OPT_{i-1}$ is a possible candidate to be the maximizing base, $M_i$. Now, with the criteria used in Algorithm \ref{alg:colsparseOT_dash} to pick the next element, we have the following.
\begin{align}
\sum_{u\in M_i}\mathbf{g}^+(u|S_{i-1}) \geq \sum_{u\in OPT_{i-1}} \mathbf{g}^+(u|S_{i-1}) \implies \sum_{u\in M_i}\left(\mathbf{g}^+(u|S_{i-1})\right)^2 & \geq \sum_{u\in OPT_{i-1}} \left(\mathbf{g}^+(u|S_{i-1})\right)^2.\nonumber
\end{align}
Using Lemma \ref{lemma2.1} and Lemma \ref{lemma2.2}, we have that, $$\sum_{u\in M_i} 2\tilde{U}_1 F(u|S_{i-1})\geq 2u_{\bar{m}} F(OPT_{i-1}|S_{i-1})\implies \sum_{u\in M_i} F(u|S_{i-1})\geq \frac{u_{\bar{m}}}{\tilde{U}_1} F(OPT_{i-1}|S_{i-1}),$$
where $\bar{m} = |OPT_{i-1}| + |S_{i-1}|$.

We denote $\frac{u_{2K}}{\tilde{U}_1}(\leq \frac{u_{\bar{m}}}{\tilde{U}_1})$ by $\alpha$. Algorithm \ref{alg:colsparseOT_dash} adds a uniformly random element $u_i\in M_i$ to the set $S_{i-1}$ to obtain the set $S_i$. As $M_i$ is of the size $K-i+1$,
\begin{align*}
\E[F(S_i)]  = F(S_{i-1}) + \E[F(u_i|S_{i-1})]  & = F(S_{i-1}) + \frac{1}{K-i+1} \sum_{u\in M_i}F(u|S_{i-1})\\
& \geq F(S_{i-1}) + \frac{\alpha}{K-i+1}F(OPT_{i-1}|S_{i-1})\\ & = F(S_{i-1}) + \alpha\frac{F(OPT_{i-1}\cup S_{i-1})-F(S_{i-1})}{K-i+1}. 
\end{align*}
\end{proof}
We now discuss the proof of Lemma \ref{lemma:struc-dash}.
\begin{proof}
Let $OPT$ be an arbitrary optimal solution. As $F(\cdot)$ is monotone, we may assume $OPT$ is a base of $\calM$. \citet[Lemma 2.2]{pmlr-v80-chen18b} describes constructing a random set $OPT_i$ for which $S_i\cup OPT_i$ is a base, for every $0\leq i \leq K$. From \citet[Lemma 2.3]{pmlr-v80-chen18b}, we have that for every $0\leq i \leq K, ~ \E[F(OPT_i)]\geq \left[ 1-\left( \frac{i+1}{K+1}\right)^{\alpha}\right] F(OPT)$. This result uses the non-negativity of $F$ and the property that $OPT_i$ is a uniformly random subset (of size $K-i$) of $OPT$.

Combining this result with Lemma \ref{matroid}, \citet[Corollary 2.5]{pmlr-v80-chen18b} gives us that for every $1\leq i\leq K, ~ \E[F(S_i)]\geq \E[F(S_{i-1})] + \alpha \frac{\{1-[i/(K+1)]^\alpha\}F(OPT)-\E[F(S_{i-1})]}{K-i+1}$. Now, the proof for the approximation ratio of Algorithm \ref{alg:colsparseOT_dash} follows from \citet[Theorem 2.6]{pmlr-v80-chen18b}. The proof is based on induction. 
\end{proof}
In the above proof, we refer the results of \citep{pmlr-v80-chen18b} as given in their arXiv version (\url{https://arxiv.org/pdf/1707.04347}). 

\subsection{Proofs of Lemma \ref{lemma:dual} and Proposition \ref{prop:dualitygap}}\label{app:dual}
\begin{proof}
We begin by re-stating the primal problem.
\begin{equation}
    \begin{array}{l}
        \min\limits_{\bgamma\geq \bzero} P(\bgamma)\big(\coloneqq \langle\mathbf{C},\bgamma\rangle + \sum_{j=1}^n \Theta(\bgamma_j)
        + \lambda_1(\|\bgamma\bone-\bmu\|_{\bG_1}^2+\|\bgamma^\top\bone-\bnu\|_{\bG_2}^2)\big),\label{app:primal}
    \end{array}
\end{equation}
where $\Theta(\bgamma_j)=\frac{\lambda_2}{2}\|\bgamma_j\|^2 + \delta_{B_K}(\bgamma_j)$ and $B_K = \{\bz\in \R_{+}^m:\|\bz\|_0\leq K\}$.
We use auxiliary variables $\p\in\R^m$ and $\q\in\R^n$ to set $\bgamma\bone-\bmu=\mathbf{p}$ and $\bgamma^\top\bone-\bnu=\mathbf{q}$. The Lagrangian becomes
\begin{equation*}
    \begin{array}{l}
        \min\limits_{\bgamma\geq \bzero; \p\in \R^m; \q \in \R^n} \ \langle\mathbf{C},\bgamma\rangle + \sum_{j=1}^n \Theta(\bgamma_j) + \lambda_1(\|\p\|_{\bG_1}^2+\|\q\|_{\bG_2}^2) + \balpha^\top(\p-\bgamma \bone + \bmu) + \bbeta^\top(\q-\bgamma^\top \bone + \bnu),
    \end{array}
\end{equation*}
where $\balpha\in \R^m$ and $\bbeta \in \R^n$. We simplify the Lagrangian as follows.
\begin{align}
    & \min\limits_{\bgamma\geq \bzero; \p\in \R^m; \q \in \R^n} \ \langle\mathbf{C},\bgamma\rangle + \sum_{j=1}^n \Theta(\bgamma_j) + \lambda_1(\|\p\|_{\bG_1}^2+\|\q\|_{\bG_2}^2) + \balpha^\top(\p-\bgamma \bone + \bmu) + \bbeta^\top(\q-\bgamma^\top \bone + \bnu) \nonumber \\
    =& \sum_{j=1}^n \min\limits_{\bgamma_j\geq \bzero} \Big(\langle\mathbf{C}_j-\balpha -\bbeta_j\bone , \bgamma_j\rangle +  \Theta(\bgamma_j) \Big) + \min\limits_{\p\in \R^m; \q \in \R^n}\Big( \lambda_1(\|\p\|_{\bG_1}^2+\|\q\|_{\bG_2}^2) + \balpha^\top(\p + \bmu) + \bbeta^\top(\q + \bnu)\Big) \nonumber \\
    =& \sum_{j=1}^n - \Theta^*(\balpha + \bbeta_j\bone -\bC_j) + \min\limits_{\p\in \R^m; \q \in \R^n}\Big(\lambda_1(\|\p\|_{\bG_1}^2+\|\q\|_{\bG_2}^2) + \balpha^\top(\p + \bmu) + \bbeta^\top(\q + \bnu)\Big).\label{lag}
\end{align}
From the optimality conditions, we have that $2\lambda_1\bG_1\p+\balpha = 0$ and $2\lambda_1\bG_2+\bbeta=0$. On substituting the values of $\p$ as $\bgamma\bone -\bmu$ and $\q$ as $\bgamma^\top \bone -\bnu$, we prove Proposition \ref{prop:dualitygap}. Using this relationship in equation (\ref{lag}), the simplified Lagrangian is denoted by $D(\balpha,\bbeta)$. The dual problem of (\ref{app:primal}) then becomes the following.
\begin{align*}
     \max\limits_{\balpha\in\R^m,\bbeta\in\R^n} D(\balpha,\bbeta)
     =\max\limits_{\balpha\in\R^m,\bbeta\in\R^n} \langle\balpha,\bmu\rangle + \langle\bbeta,\bnu\rangle - \frac{1}{4\lambda_1}\balpha^\top \bG_{1}^{-1}\balpha
         -\frac{1}{4\lambda_1} \bbeta^\top \bG_{2}^{-1}\bbeta -\sum_{j=1}^n \Theta^*(\balpha+\bbeta_j\bone - \bC_j).
\end{align*}
This proves Lemma \ref{lemma:dual}.
\end{proof}

\section{MMD-UOT problem with the support set of the variable $\bgamma$ restricted to a given set $S$}\label{app:solver}
We first present the MMD-UOT formulation in which the support set of $\bgamma$ is $T=\{(i,j)| i\in [m], j\in [n]\}$, i.e., no sparsity constraints:
$$\min_{\bgamma \geq 0}\ \mathcal{U}(\bgamma), \ \ \textup{where} $$ 
\begin{equation}
\begin{array}{ll}
\mathcal{U}(\bgamma)\coloneqq &\sum\limits_{(i,j)\in T} \big[ \bC_{ij}\bgamma_{ij} + \frac{\lambda_2}{2}\bgamma_{ij}^2 + \lambda_1 \bgamma_{ij} \big( \sum\limits_{(p,q)\in T} \bgamma_{pq} (\bG_1)_{ip} - 2(\bG_1\bmu 1_{n}^{\top})_{ij} \big) \\
&\qquad+ \lambda_1 \bgamma_{ij} \big( \sum\limits_{(p,q)\in T} \bgamma_{pq}(\bG_2)_{qj} - 2 (1_m\bnu ^\top \bG_2)_{ij} \big) \big]+\lambda_1(\|\bmu\|^2_{\bG_1} + \|\bnu\|^2_{\bG_2}).
\end{array}
\end{equation}

In the proposed formulation (\ref{eqn:reform}), the optimization is solved with the support set of $\bgamma$ being restricted to a given set $S\subseteq T$. This equivalently implies that $\bgamma_{ij}$ (and $\bgamma_{pq}$ terms) can be set to zero for all $(i,j)\in T\setminus S$ (and $(p,q)\in T\setminus S$). Consequently, the optimization problem is only for $\bgamma_{ij}$ for $(i,j)\in S$. We denote this variable as $\bgamma_S$ in the following:
$$\min_{\bgamma:\Supp(\bgamma)\subseteq S, \bgamma \geq 0}\ \mathcal{U}(\bgamma) \equiv \min_{\bgamma_S \geq 0}\ \mathcal{U}(\bgamma_S), \ \ \textup{where} $$ 
\begin{equation}
\begin{array}{ll}
\mathcal{U}(\bgamma_S)\coloneqq &\sum\limits_{(i,j)\in S} \big[ \bC_{ij}\bgamma_{ij} + \frac{\lambda_2}{2}\bgamma_{ij}^2 + \lambda_1 \bgamma_{ij} \big( \sum\limits_{(p,q)\in S} \bgamma_{pq} (\bG_1)_{ip} - 2(\bG_1\bmu 1_{n}^{\top})_{ij} \big)\\
&\qquad+ \lambda_1 \bgamma_{ij} \big( \sum\limits_{(p,q)\in S} \bgamma_{pq}(\bG_2)_{qj} - 2 (1_m\bnu ^\top \bG_2)_{ij} \big) \big] + \lambda_1(\|\bmu\|^2_{\bG_1} + \|\bnu\|^2_{\bG_2}).
\end{array}
\end{equation}

We note that the above problem has a smooth convex quadratic objective with non-negativity constraint and, therefore, can be  solved using the APGD solver \citep{mmd-uot}.

\section{Classical greedy and non-stochastic OMP algorithms for solving our GenSparseUOT (\ref{eqn:gensparse})}\label{app:gensparseuotalgo}
Algorithm~\ref{alg:gensparseOT_greedy} is the classical greedy algorithm for solving the proposed GenSparseUOT formulation  (\ref{eqn:gensparse}). 
\begin{algorithm}[tb]
\caption{Classical greedy algorithm for maximizing (weakly) submodular problems with cardinality constraint}
\begin{algorithmic}\label{alg:gensparseOT_greedy}
\STATE \textbf{Input:} $\lambda_1, \lambda_2, \bmu, \bnu, \bC, \bG_1, \bG_2$, sparsity level $K$.
\STATE $i=1, S_0=\phi, \bgamma_{S_0}=\bzero$.
\WHILE {$i\leq K$}
\STATE \textbf{1.} $u = \argmax_{e\in V\setminus S_{i-1}}F(S_{i-1}\cup \{e\})-F(S_{i-1})$
\STATE \textbf{2.} $S_i = S_{i-1}\cup \{u\}$
\STATE \textbf{3.} $\bgamma_{S_i} = \argmin_{\bgamma:\Supp(\bgamma)\in S_i,~ \bgamma\ge \bzero}{\calU(\bgamma)}$ 
\STATE \textbf{4.} $i=i+1$
\ENDWHILE
\STATE \textbf{return} $S_K, \bgamma_{S_K}$
\end{algorithmic}
\end{algorithm}

Algorithm~\ref{alg:gensparseOT_dash} is a non-stochastic OMP algorithm for for maximizing weakly sub- modular problems with cardinality constraint. It is used for solving GenSparseUOT sub-problems in the SPFD experiments (Section~\ref{exp-top}, Problem~(\ref{eqn:proposed_spfd})). 

\begin{algorithm}[tb]
\caption{Vanilla OMP algorithm for maximizing weakly submodular problems with cardinality constraint}\label{alg:gensparseOT_dash}
\begin{algorithmic}
\STATE \textbf{Input:}  $\lambda_1, \lambda_2, \bmu, \bnu, \bC, \bG_1, \bG_2$, sparsity level $K$.
\STATE $i=1, S_0=\phi, \bgamma_{S_0}=\bzero$ and $\bg = -\nabla \calU(\bgamma_{S_0})$.
\WHILE {$i\leq K$}
\STATE \textbf{1.} $u = \argmax_{e\in V\setminus S_{i-1}}{\mathbf{g}_e}$
\STATE \textbf{2.} $S_i = S_{i-1}\cup \{u\}$
\STATE \textbf{3.} $\bgamma_{S_i} = \argmin_{\bgamma:\Supp(\bgamma)\in S_i,~ \bgamma\ge \bzero}{\calU(\bgamma)}$ 
\STATE \textbf{4.} $\mathbf{g} = -\nabla \calU(\bgamma_{S_i})$
\STATE \textbf{5.} $i=i+1$
\ENDWHILE
\STATE \textbf{return} $S_K, \bgamma_{S_K}$
\end{algorithmic}
\end{algorithm}


\section{More on Experiments}\label{app:exp}
We present details of experiments discussed in the main paper along with some additional results. 

\textbf{Common Experimental Details:} In the proposed approach, we either use the RBF kernel $k(x, y) = \exp{\frac{-\|x-y\|^2}{2\sigma^2}}$ or kernels from the inverse multiquadratic (IMQ) family: $k(x, y) = (\sigma^2+\|x-y\|^2)^{-0.5}$ (referred to as IMQ) and $k(x, y) = \left(\frac{1+\|x-y\|^2}{\sigma^2}\right)^{-0.5}$ (referred to as IMQ-v2). These are the universal kernels \citep{SriperumbudurFL11, NIPS2017_dfd7468a, cadgan, dwivedi2022generalized, mmd-uot, cot}. The cost function is squared-Euclidean unless otherwise mentioned. We also normalize the cost matrix such that all entries are upper-bounded by 1. The coefficient of quadratic regularization $\lambda_2$ is 0 unless otherwise mentioned.

The experiments in $\S$ \ref{colsparseexp} are done on an NVIDIA A100-SXM4-40GB GPU, and the remaining experiments are done on an NVIDIA GeForce RTX 4090 GPU.

\subsection{Synthetic Experiments}

\begin{figure*}[t]
\begin{center}
\centerline{
\includegraphics[width=\textwidth]{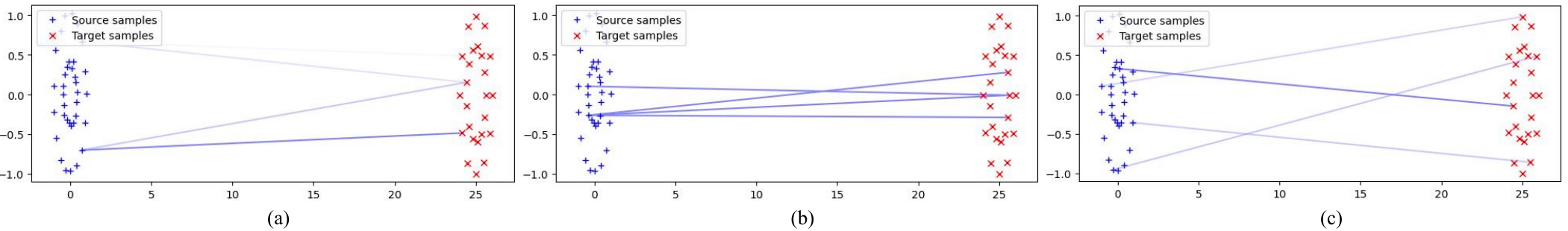}}
\caption{(Best viewed in color.) 
The source samples are shown in blue and the target samples are shown in red. We show an edge between source point $i$ and target point $j$ if $\bgamma_{i, j}>0$. The intensity of the color represents the magnitude of $\bgamma_{i, j}$. (a) SSOT \cite{Blondel2018} results in 4 non-zero entries in $\bgamma$. (b) The top-4 entries of the MMD-UOT transport plan. (c) Proposed GenSparseUOT transport plan obtained with sparsity constraint $K=4$. We can see that the support points of the transport plan obtained by GenSparseUOT are the most diverse, resulting in one-to-one mapping between the source and the target.}
\label{synthF1}
\end{center}
\vskip -0.2in
\end{figure*}

\begin{figure*}[t]
\begin{center}
\centerline{
\includegraphics[width=\textwidth]{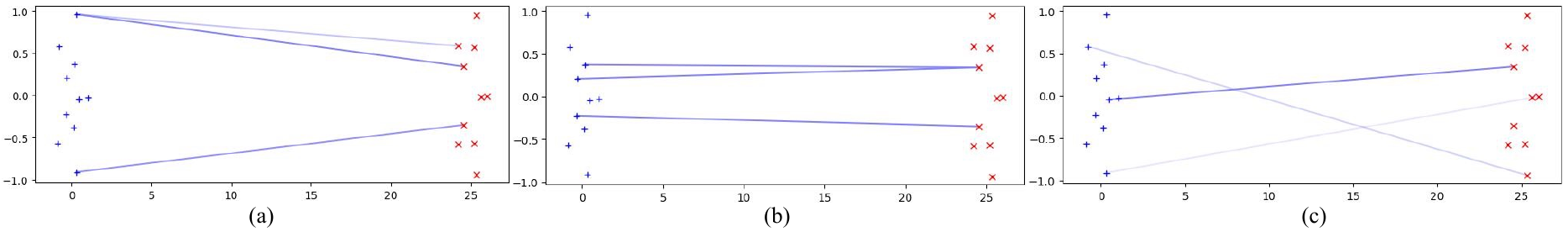}}
\caption{(Best viewed in color.) 
The source samples are shown in blue and the target samples are shown in red. We show an edge between source point $i$ and target point $j$ if $\bgamma_{i, j}>0$. The intensity of the color represents the magnitude of $\bgamma_{i, j}$. (a) SSOT \cite{Blondel2018} results in 3 non-zero entries in $\bgamma$. (b) The top-3 entries of the MMD-UOT transport plan. (c) Proposed GenSparseUOT transport plan obtained with sparsity constraint $K=3$. We can see that the support points of the transport plan obtained by GenSparseUOT are the most diverse, resulting in one-to-one mapping between the source and the target.}
\label{synthF1-10}
\end{center}
\vskip -0.2in
\end{figure*}

\textbf{Diversity in the support set.} A key feature of submodular maximization is obtaining a diverse solution set \cite{DasKempe18} since the selection of the next element essentially involves incremental gain maximization. Hence, we expect the support set of the transport plan learned by our Algorithm \ref{alg:stoch-dash} for the GenSparseUOT problem (\ref{eqn:gensparse}) to exhibit diversity. Diversity in the support set of a transport plan implies primarily learning one-to-one mappings between the source and the target points rather than one-to-many or many-to-one mappings. 


In Figure~\ref{synthF1}, we observe the diversity in the learned transport plan with $K=4$ on the two-dimensional source and target sets. 
%
%
Figure~\ref{synthF1}(a) shows the transport plan obtained using SSOT \cite{Blondel2018}. We observe that SSOT learns many-to-many mappings. 
For Figure~\ref{synthF1}(b), we observe that MMD-UOT \citep{mmd-uot} also has similar issues. It should be noted that both SSOT and MMD-UOT do not provide a direct control over the size of support of the transport plan. Hence, one may require a top-K selection heuristic to learn a transport plan with $K$ non-sparse entries \citep{arase-etal-2023-unbalanced}. 
However, as discussed, proposed Algorithm~\ref{alg:stoch-dash} directly learns a transport plan $\bgamma$ with $K$ non-sparse entries. In addition, as observed in Figure \ref{synthF1}(c), Algorithm~\ref{alg:stoch-dash} learns several one-to-one mappings, highlighting the diversity in the support set of $\bgamma$. 

\textbf{Gradient flow.} Gradient flow constructs the trajectory of a source distribution $\bar{\bmu}$ being transformed to a given target distribution $\bar{\bnu}$. The underlying problem in gradient flow is of solving $\partial_t\bar{\bmu}_t=-\nabla_{\bar{\bmu}_t}D(\bar{\bmu}_t, \bar{\bnu})$ for different timesteps $t>0$, where $D$ is a divergence over measures. Prior works have employed an OT-based divergence \cite{fatras2019learnwass, bomb-ot} and used the Euler scheme for solving this problem \cite{pmlr-v89-feydy19a}. Often, in practice, the gradient updates are performed only over the support of the distribution, keeping the mass values of the distribution fixed to uniform \cite{bomb-ot}. We compare our approach (\ref{eqn:gensparse}) with MMD-UOT \citep{mmd-uot}. 

In our experiment, the initial source distribution and the target distributions are shown in Figure~\ref{grad-flow}(a). Both the source and the target sets have $1000$ data points each. The learning rate for gradient updates is fixed to $0.01$ and the number of iterations is set to $2450$. 
Figures~\ref{grad-flow}(b)~\&~\ref{grad-flow}(c) plot the results for MMD-UOT and the proposed GenSparseUOT formulations. We observe that the solution $\bmu_t$ obtained by GenSparseUOT closely mimics the target distribution, while the solution obtained with MMD-UOT performs poorly. This is interesting because while GenSparseUOT employs MMD-UOT based objective, the additional sparsity constraint and the submodular maximization approach (Algorithm~\ref{alg:stoch-dash}) ensures diversely selected support for GenSparseUOT. This makes the gradients with the proposed approach more informative.

\begin{figure*}
\centering{
\includegraphics[width=\columnwidth]{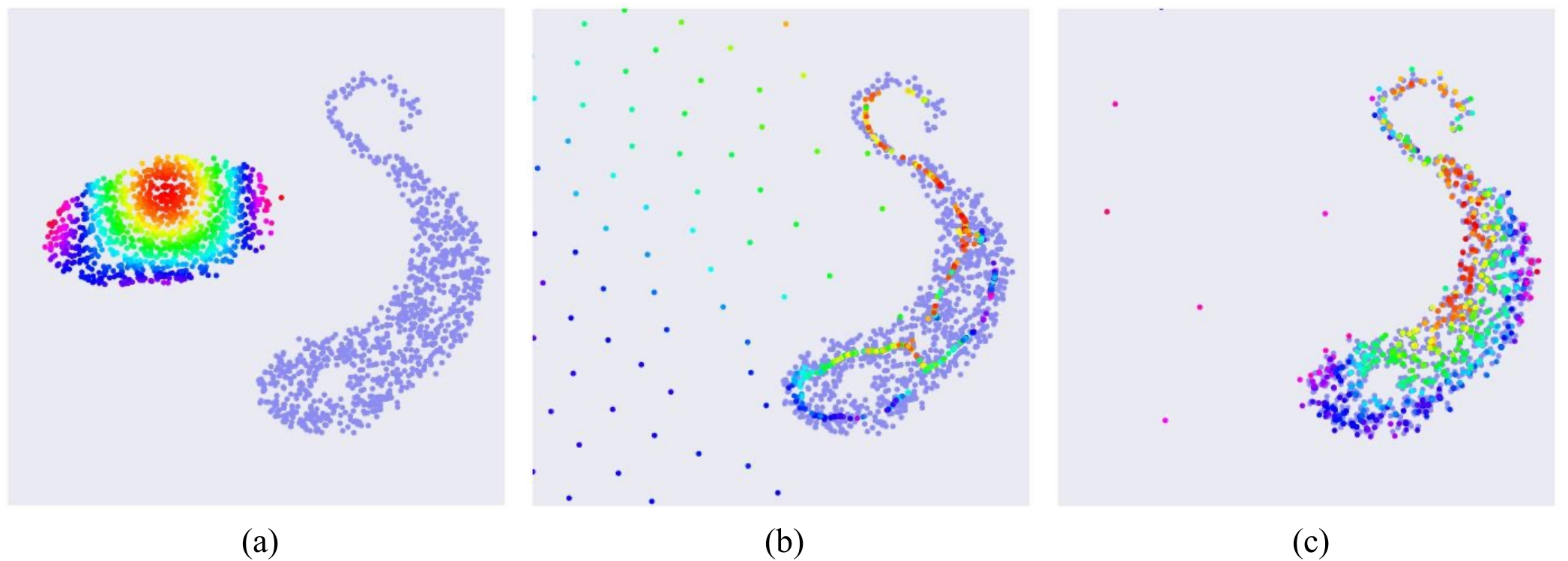}
 \caption{(Best viewed in color.) (a) Initial source points (rainbow color) on the left and target points (in blue) on the right. (b) Gradient Flow results of MMD-UOT (c) Gradient Flow results of proposed GenSparseUOT solved with Algorithm \ref{alg:gensparseOT_dash}.}
\label{grad-flow}
}
\end{figure*}

The details of the hyperparameters used for Fig. \ref{synthF1} are as follows. We consider empirical measures over the two-dimensional source and target samples with no. of source samples as 35 and no. of target samples as 25. The coefficient of regularization hyperparameter for SSOT is chosen from \{0.1, 0.5, 1\}. The result in Fig. \ref{synthF1}(a) has a coefficient of 0.5, which resulted in the OT plan with the most diverse support points. The results obtained with the proposed method and with MMD-UOT use RBF kernel with $\sigma^2$ as $1$ and $\lambda_1$ as 10. Fig. \ref{synthF1-10} shows results with empirical measures over the two-dimensional source and target samples with no. of source and target samples as 10 each. The coefficient of regularization for SSOT is 0.5, which resulted in the OT plan with the most diverse support points. The results obtained with the proposed method and with MMD-UOT use IMQ kernel with $\sigma^2$ as $10$ and $\lambda_1$ as 10.

The details of hyperparameters used for Fig. \ref{grad-flow} are as follows. For the proposed method, we use IMQ kernel with $\sigma^2$ as $10^{-4}$ and $\lambda_1$ as $10^{-1}$. For MMD-UOT, we also use the IMQ kernel but additionally validated over a range of hyperparameters: $\sigma^2\in\{10^{-4}, 10^{-3}, 10^{-2}\},~ \lambda_1\in \{10^{-4}, 10^{-3}, 10^{-2}, 10^{-1}, 1, 10\}$. The best $\sigma^2$ is $10^{-3}$ and the best $\lambda_1= 10^{-2}$.
 
\subsection{Sparse Process Flexibility Design Experiment Details ($\S~\ref{exp-top}$ of the main paper)} \label{app:spfd}
Table \ref{app-table-top} shows the detailed result where we present the mean and the standard deviation of expected profit when run for different seeds. We also show the result with our non-stochastic variant Algorithm \ref{alg:gensparseOT_dash}.

\textbf{Validation details:} The validation data split is generated following the procedure given by \citet{ijcai2023p679}. Both the GSOT's hyperparameters ($\alpha$ and $\rho$) are independently chosen from the set $\{10^{-2}, 10^{-1}, 1, 10\}$. For the proposed approach, the regularization parameter $\lambda_1$ and the kernel parameter $\sigma^2$ are chosen from the sets $\{1, 10, 100\}$ and $\{10^{-3}, 10^{-2}, 10^{-1}, 1\}$, respectively.

The hyperparameters chosen after validation are as follows. The proposed approach uses RBF kernel with $\sigma^2 =1$ and $\lambda_1=100$. The hyperparameters $(\alpha, \rho)$ in GSOT are set as $(10, 1)$. 

\textbf{Timing results:} Table \ref{table: Time T1} compares the time taken for the SPFD experiment by the GSOT method and the proposed approach. This computation time includes the time taken to compute the OT plans and the time taken to compute the overall profit as described in $\S$ \ref{exp-top}. It should be noted that the algorithm to compute the overall profit is the same for both the methods.

\begin{table}[t]
\caption{Expected profit (higher is better) for SPFD experiment with varying network size constraint $l$. Proposed refers to our GenSparseUOT formulation (\ref{eqn:gensparse}). We report the mean and std. deviation with 5 different seed values. We observe that our approach outperforms GSOT.}
\label{app-table-top}
\centering
\begin{tabular}{lccc}
\toprule
Method & $l=100$ & $l=175$ & $l=250$\\
\midrule
    GSOT & $0.014 \pm 0.001$ & $0.031 \pm 0.001$ & $0.044 \pm 0.002$ \\
    Proposed (Algorithm \ref{alg:gensparseOT_dash}) & $0.152 \pm 0.006$ & $0.212 \pm 0.011$ & $0.252\pm0.008$\\
    Proposed (Algorithm \ref{alg:stoch-dash}, $\epsilon= 10^{-2}$)  & $0.166 \pm 0.013$ & $ 0.224 \pm 0.029 $ & $ \mathbf{0.293\pm0.023}$\\
    Proposed (Algorithm \ref{alg:stoch-dash}, $\epsilon= 10^{-3}$) & $ \mathbf{0.167\pm0.017}$ & $ 0.238 \pm 0.021$ & $0.286 \pm 0.017$\\
    Proposed (Algorithm \ref{alg:stoch-dash}, $\epsilon= 10^{-4}$) & $0.147 \pm 0.018$ & $ \mathbf{0.240\pm0.015}$ & $ 0.274 \pm 0.008$\\
\bottomrule
\end{tabular}
\end{table}
 \begin{table}
\caption{Computation time (s) corresponding to Table \ref{table-top} results.}
\label{table: Time T1}
\centering
\setlength{\tabcolsep}{4pt}
\begin{tabular}{lccc}
\toprule
Method & $l=100$ & $l=175$ & $l=250$\\
\midrule
    GSOT & 17.69 & 20.09 & 23.00 \\
    Proposed ($\epsilon= 10^{-2}$) & 6.74 & 11.99 & 17.59 \\
    Proposed ($\epsilon= 10^{-3}$) & 6.33 & 12.03 & 17.93 \\
    Proposed ($\epsilon= 10^{-4}$) & 6.44 & 11.92 & 17.68 \\
\bottomrule
\end{tabular}
\end{table}

\subsection{Word Alignment Experiment Details ($\S~\ref{exp:word-alignment}$ of the main paper)} \label{app:word_alignment}
\textbf{Dataset:} The Wiki dataset consists of 2514 training instances, 533 validation instances, and 1052 test instances. 


\textbf{Experimental setup:} For baseline methods BOT, POT, and KL-UOT, the results were obtained with the code and optimal hyperparameters shared by \citet{arase-etal-2023-unbalanced}. To evaluate the proposed method, we use the same experimental setup as in \citep{arase-etal-2023-unbalanced} and only tune the hyperparameters. Cosine-distance is employed as the cost function. 

\textbf{Validation Details:} The validation data split is the same provided by \citet{arase-etal-2023-unbalanced}. The grid for regularization hyperparameters for BOT, POT, KL-UOT are the same as in their code, i.e., BOT, POT have 50 equally-spaced values between 0 and 1 and KL-UOT has 200 equally spaced values in the log space between -3 and 3. For SSOT, the regularization hyperparameter $\lambda$ is chosen from $\{10^{-7}, 10^{-6}, \ldots, 1\}$. 
For both MMD-UOT and the proposed approach: (a) the kernel function is validated between RBF and IMQ, (b) the kernel hyperparameters are tuned from the set $\{m/8, m/4, m, 4m, 8m\}$, where $m$ denotes the median used in median heuristics \citep{gretton12a}, and (c) $\lambda_1$ is tuned from the set $\{0.1, 1, 10\}$. For the proposed method, $\lambda_2$ is validated on the set $\{0.1,0\}$.

The chosen hyperparameters for SSOT, MMD-UOT, and our approach are as follows. The coefficient of $\ell_2$-norm regularization $\lambda$ for SSOT is $10^{-4}$. 
For MMD-UOT, IMQ kernel is chosen with $\sigma^2=4m$ and $\lambda_1$ as 10. 
For our approach, the chosen kernel is RBF with $\sigma^2=m/8$, $\lambda_1=1$, and $\lambda_2=0.1$. 

\textbf{Results:} Table \ref{table-word} reports the F1 and accuracy scores while Table \ref{table-waln2} reports the corresponding precision and recall scores. 

\textbf{Timing results:} The average time (in seconds) to compute OT plans: (a) BOT, POT, and KL-UOT baselines of \citet{arase-etal-2023-unbalanced} require 0.01 seconds, (b) SSOT requires 0.08 seconds, (c) MMD-UOT requires 0.40 seconds, and (d) our approach requires 1.06 seconds. 

\begin{table}[t]
\caption{Precision and Recall values on the test split of the Wiki dataset. Higher values are better.}
\label{table-waln2}
\centering{
\begin{tabular}{lcccc}
\toprule
& \multicolumn{2}{c}{Null} & \multicolumn{2}{c}{Total}\\
Method & Precision & Recall & Precision & Recall\\
\midrule
    BOT \cite{arase-etal-2023-unbalanced} & \textbf{79.80} &	{80.29} & {95.11} &	\textbf{94.81} \\
    POT \cite{arase-etal-2023-unbalanced} & 66.96 &	79.01  & 93.64 &	94.67\\
    KL-UOT \cite{arase-etal-2023-unbalanced} & 77.31 &	80.16 &	 94.50 &	{94.75} \\
    MMD-UOT \cite{mmd-uot} & 76.41 & 75.42 & 92.73 &	93.57 \\
    SSOT \cite{Blondel2018} & 23.02 &	40.64  & 57.71 & 72.15 \\
    Proposed & {79.14}	& \textbf{80.72} &	 \textbf{95.13} &	94.45 \\
\bottomrule
\end{tabular}
}
\end{table}



\subsection{Column-wise Sparse Transport Plan Experiment Details ($\S~\ref{colsparseexp}$ of the main paper)}\label{colsp}
\subsubsection{Toy Experiment}
\begin{figure}
\centering{
    \includegraphics[scale=0.5]{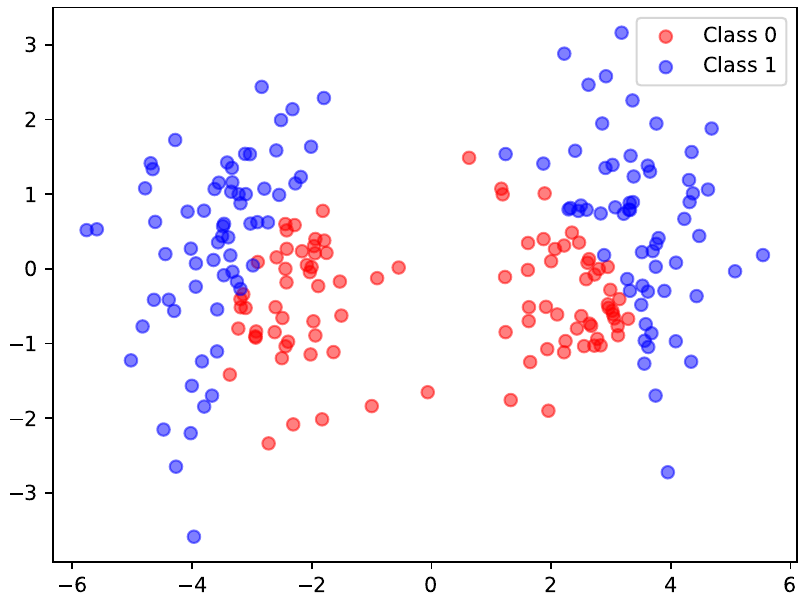}
    \caption{Synthetic data (best viewed in color) for the toy experiment in $\S~\ref{colsparseexp}$.}\label{fig:data-xor}
    }
\end{figure}
Figure \ref{fig:data-xor} shows the data generated for this experiment. Each of the experts is a 1-hidden-layer neural network with GELU as the activation function. The gating function is a single-layer neural network. We employ the Adam optimizer with a constant learning rate scheduler and train for 100 epochs. 

\textbf{Validation Details:} The randomly sampled validation split consists of $15\%$ instances. The regularization hyperparameter for SCOT is from $\{10^{-2}, 10^{-1}, 1, 10, 100\}$. For the proposed approach, $\lambda_1$ is chosen from $\{10^{2}, 10, 1\}$ and $\lambda_2$ is chosen from $\{10, 1\}$. The validation set for other baselines are as follows: 
(a) regularization hyperparameter for SSOT: $\{10^{-2}, 10^{-1}, 1, 10 \}$, (b) KL-UOT: marginal's regularization $\{10^{-1}, 1, 10 \}$, entropic regularization $\{10^{-2}, 10^{-1}, 1, 10\}$, and (c) regularization hyperparameter for entropic OT $\{10^{-2}, 10^{-1}, 1, 10 \}$. 

The hyperparameter chosen for the proposed approach: IMQ-v2 kernel with $\sigma^2=100$ and regularization parameters $\lambda_1=100, \lambda_2=10$. The coefficient of regularization chosen for SCOT is $\lambda=10$.

\subsubsection{CIFAR-10 vs CIFAR-10-ROT}
We follow the default setting of the code provided by \citet{chen2022towards}.

For the proposed method, we fix $\lambda_2=10$, the kernel function as IMQ-v2, and we set the kernel hyperparameter following the median-heuristics \citep{gretton12a}. The $\ell_2$-norm regularization hyperparameter $\lambda$ of SCOT and the marginal regularization hyperparameter $\lambda_1$ of the proposed approach are chosen $\{0.1, 10, 1000\}$. 
The default coefficient of the load-balancing loss taken from the code by \citet{chen2022towards} is $n^2$, where $n$ is the number of experts. The default setting results in a skewed allocation across the experts. On increasing the coefficient of the load-balancing loss to $n^8$, we get a more balanced split of inputs across the experts.
Other hyperparameters are set to the default value in the code by \citet{chen2022towards}. 
The test data consists of 20000 examples, with 10000 examples each from CIFAR-10 and the CIFAR-10 rotated. 

\textbf{Timing results:} The per-epoch computation time in seconds corresponding to Table~\ref{table-cifar} are:  58.75s for (vanilla) MoE, 59.92s for SCOT, and  617.90s for our approach.


\begin{table}[t]
\caption{Duality gap comparison for solving Problem (\ref{eqn:primal}) on varying the regularization hyperparameters. All values are rounded to 6 decimal places. The kernel used is IMQ. A lower duality gap is better.}
\label{table: DG-imq}
\centering
\setlength{\tabcolsep}{4pt}
\begin{tabular}{rrcccc}
\toprule
$\lambda_1$ & $\lambda_2$ & \multicolumn{2}{c}{Proposed solver} & \multicolumn{2}{c}{SCOT solver}\\
 &  & Primal obj. & Duality Gap & Primal obj.  & Duality Gap \\ 
\midrule
0.1 & 0.1 & 0.006073 &	$\mathbf{< 10^{-10}}$ & 0.006073 & 0.000020 \\
1 & 0.1 & 0.040079	& $\mathbf{0.000014}$ & 0.060187 	& 0.021549\\
10 & 0.1 & 0.090064	& $\mathbf{0.015801}$	& 0.502088		& 0.418079 \\
0.1 & 1 & 0.006073	& $\mathbf{< 10^{-10}}$	& 0.006073	& 0.000417 \\
1 & 1 & 0.042633 &	$\mathbf{0.000012}$	& 0.043374	& 0.001185 \\
10 & 1 & 0.092715 & $\mathbf{0.001890}$	& 0.095961 & 0.005033 \\ 
\bottomrule
\end{tabular}
\end{table}

\subsection{Duality Gap Comparison Experiment Details (Section~\ref{subsec:dualitygapcomparison} of the main paper)}\label{app:dualitygap}
We present a comparison of the duality gaps obtained using the proposed solver (Algorithm~\ref{alg:colsparseOT_dash}) and the SCOT-based solver for optimizing (\ref{eqn:primal}). 

\textbf{Adapting the SCOT solver for solving the dual problem (\ref{eqn:dual}) corresponding to the primal problem (\ref{eqn:primal}):} 
Following \citet{liu2023sparsityconstrained}, we use the LBFGS optimizer from \texttt{scipy.optimize} and initialize the dual variables $\balpha, ~\bbeta$ as zero vectors of appropriate dimensions. Using the LBFGS optimizer requires one to pass a module that takes in inputs as the optimization variable ($\balpha, ~\bbeta$ in our case) and returns the objective value in (\ref{eqn:dual}) along with the expression for the gradient of the objective w.r.t. the optimization variables. The gradient of the dual objective w.r.t. $\balpha$ is $\bmu-\frac{\bG^{-1}_1\balpha}{2\lambda_1}-\mathbf{z}\bone$ and the gradient w.r.t. $\bbeta$ is $\bnu - \frac{\bG_2^{-1}\bbeta}{2\lambda_1}-\mathbf{z}^\top \bone$, where $\bz$ is the solution of the sparse projection problem (\ref{eqn:conjugate})  \citep{liu2023sparsityconstrained}. 

We set max-iter for the APGD algorithm used in the proposed solver (Algorithm~\ref{alg:colsparseOT_dash}) as 1000. 
The results with the SCOT solver are also reported with 1000 as the maximum iteration (after confirming that a higher max-iter does not change the duality gap). 

\textbf{Results:} Tables \ref{table: DG-imq} 
and \ref{table: DG-rbf} show the duality gaps associated with the proposed solver and the SCOT solver with RBF and IMQ kernels, respectively, and over different hyperparameter values. 
Table~\ref{table: main-paper-DG-imq-v2} in the main paper shows duality gaps with the solvers employing IMQ-v2 kernel. 
The kernel hyperparameter is fixed according to the median heuristics \cite{gretton12a}, and the column-wise sparsity level is fixed as $K_2=4$. 
We observe that the proposed solver obtains better duality gaps across regularization hyperparameters and kernels.


\begin{table}
\caption{Duality gap comparison for solving Problem (\ref{eqn:primal}) on varying the regularization hyperparameters. All values are rounded to 6 decimal places. The kernel used is RBF. A lower duality gap is better.}
\label{table: DG-rbf}
\centering
\setlength{\tabcolsep}{4pt}
\begin{tabular}{rrcccc}
\toprule
$\lambda_1$ & $\lambda_2$ & \multicolumn{2}{c}{Proposed solver} & \multicolumn{2}{c}{SCOT solver}\\
 &  & Primal obj. & Duality Gap & Primal obj. & Duality Gap \\ 
\midrule
0.1 & 0.1 & 0.002000	& $\mathbf{< 10^{-10}}$ & 0.002000	& $\mathbf{< 10^{-10}}$ \\
1 & 0.1 & 0.019944	& $\mathbf{< 10^{-10}}$ & 0.020003	& 0.000317\\
10 & 0.1 & 0.094129	& $\mathbf{0.057683	}$ & 0.174627 & 0.102349 \\
0.1 & 1 & 0.002000	&  $\mathbf{< 10^{-10}}$	& 0.002000		& $\mathbf{< 10^{-10}}$ \\
1 & 1 & 0.019953	& $\mathbf{< 10^{-10}}$	& 0.020218	& 0.000881\\
10 & 1 & 0.094866	& $\mathbf{0.008412}$	& 0.155751	& 0.064417 \\ 
\bottomrule
\end{tabular}
\end{table}
  \begin{figure}
  \centering{
     \includegraphics[width=0.5\columnwidth]{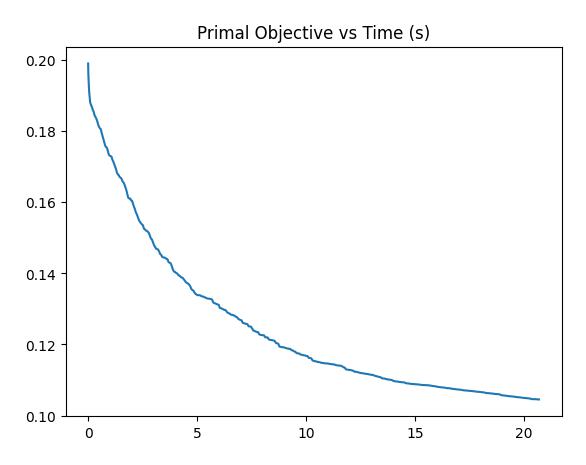}
     \caption{Primal obj vs Time (s) plot for solving the ColSparseUOT formulation using Algorithm~\ref{alg:colsparseOT_dash}. The time is computed on an Intel-i9 CPU.}\label{comp}
     }
 \end{figure}

\subsection{Computation Time}
Figure \ref{comp} shows the objective over time (s) plot while computing column-wise sparse transport plan using Algorithm \ref{alg:colsparseOT_dash}. The source and target measures are empirical measures over two randomly chosen 100-sized batches of CIFAR-10. The kernel used is RBF with median heuristics \cite{gretton12a}. The sparsity level $K_2$ is 4 and $\lambda_1=\lambda_2=10$. The computation is done on an Intel-i9 CPU.


\subsection{Color Transfer Experiment}
Following \citet{Blondel2018}, we perform an experiment of OT-based color transfer. Figure~\ref{color-cat} shows the results with various methods and the sparsity level in the obtained transport map. The coefficient of $\ell_2$-norm regularization hyperparameter for SSOT is 1. For the proposed method and for MMD-UOT, we use RBF kernel with $\sigma^2=10^{-2}$ and $\lambda_1=0.1$.
  \begin{figure}[t]
  \centering{
     \includegraphics[width=\columnwidth]{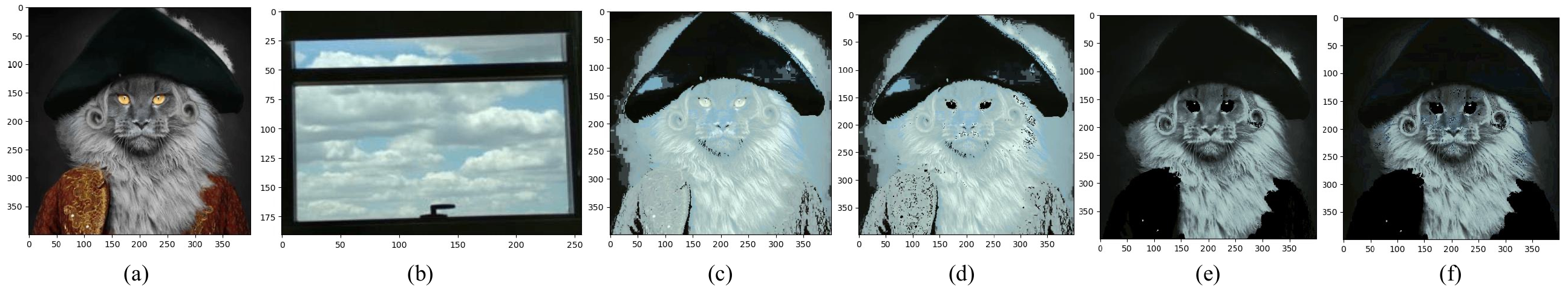}
     \caption{(Best viewed in color.) (a) Image 1 (b) Image 2. (c)-(f) show results (along with the level of sparsity) obtained by transferring the colors from Image 2 to Image 1: (c) OT ($99.22\%$) \cite{KatoroOT} (d) SSOT ($97.86\%$) \cite{Blondel2018} (e) MMD-UOT ($96.12\%$) \cite{mmd-uot} (f) Proposed GenSparseUOT ($99.61\%$).}
     \label{color-cat}
     }
 \end{figure}



\end{document}